\documentclass{article}

     \usepackage[preprint]{neurips_2019}

\usepackage[utf8]{inputenc} % allow utf-8 input
\usepackage[T1]{fontenc}    % use 8-bit T1 fonts
\usepackage{hyperref}       % hyperlinks
\usepackage{url}            %  URL typesetting
\usepackage{booktabs}       % professional-quality tables
\usepackage{amsfonts}       % blackboard math symbols
\usepackage{nicefrac}       % compact symbols for 1/2, etc.
\usepackage{microtype}      % microtypography
\usepackage{amsmath, amsthm}
\usepackage{algorithm, algorithmicx, algpseudocode}
\usepackage{graphicx}
\usepackage{subfigure}
\usepackage{tikz}

\theoremstyle{plain}
\newtheorem{theorem}{Theorem}
\newtheorem{lemma}{Lemma}

\newtheorem{proposition}{Proposition}

\theoremstyle{definition}
\newtheorem{definition}{Definition}
\newtheorem{assumption}{Assumption}

\theoremstyle{remark}
\newtheorem{remark}{Remark}

\title{Feature-Based Q-Learning for Two-Player Stochastic Games}

\author{%
  Zeyu Jia \\
  Peking University\\
  \texttt{jiazy@pku.edu.cn} \\
  \And
  Lin F. Yang\\
  Princeton University\\
  \texttt{lin.yang@princeton.edu}\\
  \And
  Mengdi Wang\\
  Princeton University\\
  \texttt{mengdiw@princeton.edu}
}

\newcommand{\pr}{\mathbf{Pr}}
\newcommand{\mM}{\mathcal{M}}
\newcommand{\mE}{\mathcal{E}}
\newcommand{\mG}{\mathcal{G}}
\newcommand{\mA}{\mathcal{A}}
\newcommand{\mS}{\mathcal{S}}
\newcommand{\mK}{\mathcal{K}}
\newcommand{\mT}{\mathcal{T}}

\begin{document}
\maketitle

\begin{abstract}
	\par Consider a two-player zero-sum stochastic game where the transition function can be embedded in a given feature space. We propose a two-player Q-learning algorithm for approximating the Nash equilibrium strategy via sampling. The algorithm is shown to find an $\epsilon$-optimal strategy using sample size linear to the number of features. To further improve its sample efficiency, we develop an accelerated algorithm by adopting techniques such as variance reduction, monotonicity preservation and two-sided strategy approximation. We prove that the algorithm is guaranteed to find an $\epsilon$-optimal strategy using no more than $\tilde{\mathcal{O}}(K/(\epsilon^{2}(1-\gamma)^{4}))$ samples with high probability, where $K$ is the number of features and $\gamma$ is a discount factor. The sample, time and space complexities of the algorithm are independent of original dimensions of the game.
\end{abstract}

\section{Introduction}
\par Two-player turn based stochastic game (2-TBSG) is a generalization of Markov decision process (MDP), both of which are widely used models in machine learning and operations research. While MDP involves one agent with its simple objective to maximize the total reward, 2-TBSG is a zero-sum game involving two players with opposite objectives: one player seeks to maximize the total reward and the other player seeks to minimize the total reward. In a 2-TBSG, the set of all states is divided into two subsets that are controlled by the two players, respectively.  
We focus on the discounted stationary 2-TBSG, where the probability transition model is invariant across time and the total reward is the infinite sum of all discounted rewards. Our goal is to approximate the Nash equilibrium of the 2-TBSG, whose existence is proved in \cite{shapley1953stochastic}.
\par There are two practical obstacles standing in solving 2-TBSG:
\begin{itemize}
		\vspace{-2mm}
	\item We usually do not know the transition probability model explicitly;
		\vspace{-1.5mm}
	\item The number of possible states and actions are very large or even infinite.
		\vspace{-2mm}
\end{itemize}
In this paper we have access to a sampling oracle that can generate sample transitions from any state and action pair. We also suppose that a finite number of state-action features are available, such that the unknown probability transition model can be embedded using the feature space. These features allow us to solve 2-TBSG of arbitrary dimensions using parametric algorithms. 

\par A question is raised naturally, that is, how many samples are needed in order to find an approximate Nash equilibrium? For solving the one-player MDP to $\epsilon$-optimality using $K$ features, \cite{yang2019sample} proves an information-theoretic lower bound of sample complexity $\Omega(K/((1-\gamma)^{3}\epsilon^{2}))$. Since MDP is a special case of 2-TBSG, the same lower bound applies to 2-TBSG. Yet there has not been any provably efficient algorithm for solving 2-TBSG using features.

\par 
To answer this question, we propose two sampling-based algorithms and provide sample complexity analysis. 
Motivated by the value iteration and Q-learning like algorithms given by \cite{hansen2013strategy,yang2019sample}, we propose a two-player Q-learning algorithm for solving 2-TBSG using given features. When the true transition model can be fully embedded in the feature space without losing any information, our algorithm finds an $\epsilon$-optimal strategy using no more than $\tilde{O}(K/((1-\gamma)^{7}\epsilon^{2}))$ sample transitions, where $K$ is the number of state-action features. We also provide model misspecification error bound for the case where the features cannot fully embed the transition model.

\par To further improve the sample complexity, we use a variance reduction technique, together with a specifically designed monotonicity preservation technique which were previously used in \cite{yang2019sample}, to develop an algorithm that is even more sample-efficient. This algorithm uses a two-sided approximation scheme to find the equilibrium value from both above and below. It computes the final $\epsilon$-optimal strategy by sticking two approximate strategies together. This algorithm is proved to find an $\epsilon$-optimal strategy with high probability using $\tilde{O}(K/((1-\gamma)^{4}\epsilon^{2}))$ samples, which improves significantly from our first result. %Though there is still a $1/(1-\gamma)$ gap between the lower bound of sample complexity and the complexity of our algorithm, 
Our results are the first and sharpest sample complexity bounds for solving two-player stochastic game using features, to our best knowledges. Our algorithms are the first ones of their kind with provable sample efficiency. It is also worth noting that the algorithms are space and time efficient, whose complexities depend polynomially on $K$ and $\frac1{1-\gamma}$, and are free from the game's dimensions. 

%To sum up, our algorithms are sample-, space- and time efficient at 

%\par Given the parameter $\theta$ and state $s$, suppose the time cost in calculating $V_{\theta}(s)$ is $M$, which is equivalent to solving an optimization problem. Then the total time cost of our second algorithm is $\tilde{\mathcal{O}}(K^{\omega} + MK / ((1-\gamma)^{4}\epsilon^{2}) + K^{2} / (1-\gamma))$, where $K^{\omega}$ is the time cost calculating $\Phi_{\mK}^{-1}$. The space required in our algorithms is $\tilde{\mathcal{O}}(K^{2} + K/(1-\gamma))$. This time and space complexity only depend on the number of features, and is independent to the number of states and actions.

\par In Section \ref{sec2} we review  related literatures. Section \ref{sec3} presents the problem formulation and basics. We introduce a basic two-player Q-learning algorithm in Section \ref{sec4} together with its analysis. The accelerated two-player Q-learning algorithm and its analysis are presented in Section \ref{sec5} and Section \ref{sec6}.

%%%%%%%%%%%%%%%%%%%%
%%%%%%%%%%%%%%%%%%%%
 
\section{Related Works}\label{sec2}

\par The 2-TBSG is a special case of games and stochastic games (SG), which are first introduced in \cite{von2007theory} and \cite{shapley1953stochastic}. For a comprehensive introduction on SG, please refer to the books \cite{neyman2003stochastic} and \cite{filar2012competitive}. A number of deterministic algorithms have been developed for solving 2-TBSG when its explicit form is fully given, including \cite{littman1996algorithms, ludwig1995subexponential, hansen2013strategy}. For example \cite{rao1973algorithms} proposes the strategy iteration algorithm. A value iteration method is proposed by \cite{hansen2013strategy}, which is one of the motivation of our algorithm. 

\par In the special case of MDP, there exist a large body of works on its sample complexity and sampling-based algorithms. For the tabular setting (finitely many state and actions), sample complexity of MDP with a sampling oracle has been studied in \cite{kearns1999finite, azar2013minimax, sidford2018variance, NIPS2018_7765, kakade2003sample, singh1994upper, azar2011speedy}. Lower bounds for sample complexity have been studied in \cite{azar2013minimax, even2006action, azar2011reinforcement}, where the first tight lower bound $\Omega(|\mS||\mA| / (1-\gamma)^{3})$ is obtained in \cite{azar2013minimax}. The first sample-optimal algorithm for finding an $\epsilon$-optimal value is proposed in \cite{azar2013minimax}. \cite{NIPS2018_7765} gives the first algorithm that finds an $\epsilon$-optimal policy using the optimal sample complexity $\tilde{O}(|\mS||\mA|/(1-\gamma)^{3})$ for {\it all} values of $\epsilon$. 
 For solving MDP using $K$ linearly additive features, \cite{yang2019sample} proved a lower bound of sample complexity that is $\Omega(K/((1-\gamma)^{3}\epsilon^{2}))$. It also provided an algorithm that achieves this lower bound up to log factors, however, their analysis of the algorithm relies heavily on an extra ``anchor state'' assumption. In \cite{chen2018scalable}, a primal-dual method solving MDP with linear and bilinear representation of value functions and transition models is proposed for the undiscounted MDP. In \cite{jiang2017contextual}, the sample complexity of contextual decision process is studied.

\par As for general stochastic games, the minimax Q-learning algorithm and the friend-and-foe Q-learning algorithm is introduced in \cite{littman1994markov} and \cite{littman2001friend}, respectively. The Nash Q-learning algorithm is proposed for zero-sum games in \cite{hu2003nash} and for general-sum games in \cite{littman2001value, hu1999multiagent}. Also in \cite{perolat2015approximate}, the error of approximate Q-learning is estimated. In \cite{zhang2018finite}, finite-sample analysis of multi-agent reinforcement learning is provided. To our best knowledge, there is no known algorithm that solves 2-TBSG using features with sample complexity analysis.

\par There are a large number of works analyzing linear model approximation of value and Q functions, for examples \cite{tsitsiklis1997analysis, nedic2003least, lagoudakis2003least, melo2008analysis, parr2008analysis, sutton2009convergent, lazaric2012finite, tagorti2015rate}. These work mainly focus on approximating the value function or Q function for a fixed policy. The convergence of temporal difference learning with a linear model for a given policy is proved in \cite{tsitsiklis1997analysis}. \cite{melo2008analysis} and \cite{sutton2009convergent} study the convergence of Q-learning and off-policy temporal difference learning with linear function parametrization, respectively. In \cite{parr2008analysis}, the relationship of linear transition model and linear parametrized value functions is explained. It is also pointed out by \cite{yang2019sample} that using linear model for Q function is essentially equivalent to assuming that the transition model can be embedded using these features, provided that there is zero Bellman error. 

\par The fitted value iteration for MDPs or 2TBSGs, where the value function is approximated by functions in a general function space, is analyzed in \cite{munos2008finite, antos2008fitted, antos2008learning, farahmand2010error, yang2019theoretical, perolat2016softened}. In these papers, it is shown that the error is related to the Bellman error of the function space, and depends polynomially on $1/\epsilon, 1/(1-\gamma)$ and the dimension of the function space. However, only convergence is analyzed in these paper.

\section{Preliminaries}\label{sec3}
%\par In this section, we present some basic knowledge of our setting.
\paragraph{Basics of 2-TBSG}
A discounted 2-TBSG (2-TBSG for short) consists of a tuple $(\mathcal{S}, \mathcal{A}, P, r, \gamma)$, where $\mathcal{S} = \mathcal{S}_{1}\cup\mathcal{S}_{2}, \mathcal{A} = \mathcal{A}_{1}\cup\mathcal{A}_{2}$ and $\mathcal{S}_{1}, \mathcal{S}_{2}, \mathcal{A}_{1}, \mathcal{A}_{2}$ are state sets and action sets for Player 1 and Player 2, respectively. $P\in\mathbb{R}^{|\mathcal{S}\times\mathcal{A}|\times|\mathcal{S}|}$ is a transition probability matrix, where $P(s'|s, a)$ denotes the probability of transitioning to state $s'$ from state $s$ if action $a$ is used. $r\in\mathbb{R}^{|\mS|\times|\mathcal{A}|}$ is the reward vector, where $r(s, a)\in [0, 1]$ denotes the immediate reward received using action $a$ at state $s$. 
\par For a given state $s\in\mathcal{S}$, we use $\mathcal{A}_{s}$ to denote the available action set for state $s$. A value function is a mapping from $\mS$ to $\mathbb{R}$. A deterministic strategy (strategy for short) $\pi = (\pi_{1}, \pi_{2})$ is defined such that $\pi_{1}, \pi_{2}$ are mappings from $\mathcal{S}_{1}$ to $\mathcal{A}_{1}$ and from $\mathcal{S}_{2}$ to $\mathcal{A}_{2}$, respectively. 
%Moreover, each state $s\in\mathcal{S}$ matches to an action in $\mathcal{A}_{s}$. 
Given a strategy $\pi$, the {\it value function} of $\pi$ is defined to be the expectation of total discounted reward starting from $s$, i.e.,
\begin{equation}\label{value}
	V^{\pi}(s) = \mathbb{E}\left[\sum_{i=0}^{\infty}\gamma^{i}r(s_{i}, \pi(s_{i}))\Big|s_{0} = s\right],\quad \forall s\in\mS,
\end{equation}
where $\gamma\in[0, 1)$ is the discounted factor, and the expectation is over all trajectories starting from $s$.
\par Two players in a 2-TBSG has opposite objectives. While the first player seeks to maximize the value function \eqref{value}, the second player seeks to minimize it. In the following we present the definition of the equilibrium strategy.
\begin{definition}
	A strategy $\pi^{*} = (\pi_{1}^{*}, \pi_{2}^{*})$ is called a Nash equilibrium strategy (equilibrium strategy for short), if $V^{\pi_{1}, \pi_{2}^{*}}\le V^{\pi^{*}}\le V^{\pi_{1}^{*}, \pi_{2}}$ for any player 1's strategy $\pi_{1}$ and player 2's strategy $\pi_{2}$.
\end{definition}
The existence of the Nash equilibrium strategy is proved in \cite{shapley1953stochastic}. And all equilibrium strategies share the same value function, which we denote by $v^{*}\in\mathbb{R}^{|\mS|}$.
\par Notice that $v^{*}$ is the equilibrium value if and only if it satisfies the following Bellman equation \cite{hansen2013strategy}:
\begin{equation}
	v^{*} = \mT v^{*},
\end{equation}
where $\mT$ is an operator mapping a value function $V$ into another:
\begin{equation}
	\mT V(s) = \begin{cases} \max_{a\in\mA_{s}}[r(s, a) + \gamma P(\cdot|s, a)^{T}V], &\quad \forall s\in\mS_{1},\\
	\min_{a\in\mA_{s}}[r(s, a) + \gamma P(\cdot|s, a)^{T}V], &\quad \forall s\in\mS_{2}.\end{cases}
\end{equation}
\par We give definitions of $\epsilon$-optimal values and $\epsilon$-optimal strategies.
\begin{definition}
	We call a value function $V$ an $\epsilon$-optimal value, if $\|V - v^{*}\|_{\infty}\le \epsilon$.
\end{definition}

\begin{definition}
	We call a strategy $\pi = (\pi_{1}, \pi_{2})$ an $\epsilon$-optimal strategy, if for any $s\in\mS$,
	\begin{equation*}
		\max_{\overline{\pi}_{1}}\left[V^{\overline{\pi}_{1}, \pi_{2}}(s) - v^{*}(s)\right]\le \epsilon,\quad \min_{\overline{\pi}_{2}}\left[V^{\pi_{1}, \overline{\pi}_{2}}(s) - v^{*}(s)\right]\ge - \epsilon.
	\end{equation*}
\end{definition}
Since $\min_{\overline{\pi}_{2}}V^{\pi_{1}, \overline{\pi}_{2}}\le v^{*}\le \max_{\overline{\pi}_{1}}V^{\overline{\pi}_{1}, \pi_{2}}$, the above definition is equivalent to
$
		\|\min_{\overline{\pi}_{2}}V^{\pi_{1}, \overline{\pi}_{2}} - v^{*}\|_{\infty}\le \epsilon$ and $\|\max_{\overline{\pi}_{1}}V^{\overline{\pi}_{1}, \pi_{2}} - v^{*}\|_{\infty}\le \epsilon$.

\paragraph{Features and Probability Transition Model}\label{sec3.2}
Suppose we have $K$ feature functions $\phi = \{\phi_{1}, \cdots, \phi_{K}\}$ mapping from $\mS\times\mA$ into $\mathbb{R}$. For every state-action pair $(s, a)$, these features give a feature vector
\begin{equation*}
	\phi(s, a) = [\phi_{1}(s, a), \cdots, \phi_{K}(s, a)]^{T}\in\mathbb{R}^{K}.
\end{equation*}
Throughout this paper, we focus on 2-TBSG where the probability transition model can be nearly embedded using the features $\phi$ without losing any information.
\begin{definition}\label{def1}
We say that the transition model $P$ can be embedded into the feature space $\phi$, if there exists functions $\psi_{1},\ldots,\psi_{K}:\mS\mapsto\mathbb{R}$ such that
\vspace{-1mm}
\begin{equation*}
	P(s'|s, a) = \sum_{k\in [K]}\phi_{k}(s, a)\psi_{k}(s'),\quad \forall s'\in\mS,\ (s, a)\in\mS\times\mA.
\end{equation*}
\end{definition}
The preceding model is closely related to linear model for Q functions.
When $P$ can be fully embedded using $\phi$, the Q-functions belong to $span\{r, \phi\}$ so we can parameterize them using $K$-dimensional vectors.
Note that the feature representation is only concerned with the probability transition but has nothing to do with the reward function. It is pointed out by \cite{yang2019sample} that having a transition model which can be embedded into $\phi$ is equivalent to using linear Q-function model with zero Bellman error. In our subsequent analysis, we also provide approximation guarantee when $P$ cannot be fully embedded using $\phi$. 

It is worth noting that Definition \ref{def1} has a kernel interpretation. It is equivalent to that the left singular functions of $P$ belong to the Hilbert space with the kernel function $K((s,a),(s',a')) = \phi(s,a)^{T}\phi(s',a')$. Our model and method can be viewed as approximating and solving the 2-TBSG in a given kernel space.

\paragraph{Notations}
For two value functions $V_{1}, V_{2}$, we use $V_{1}\le V_{2}$ to denote $V_{1}(s)\le V_{2}(s), \forall s\in\mS$. We use $\prod_{[a, b]}V(s)$ to denote the projection of $V(s)$ into the interval $[a, b]$. The total variance (TV) distance between two distributions $P_{1}, P_{2}$ on the state space $\mS$ is defined as
$
	\|P_{1} - P_{2}\|_{TV} = \sum_{s\in\mS}|P_{1}(s) - P_{2}(s)|.
$
And we use $\tilde{\mathcal{O}}(\cdot)$ to hide log factors of $K, L, \epsilon, 1-\gamma$ and $\delta$.

\section{A Basic Two-Player Q-learning Algorithm}\label{sec4}
\par In this section, we develop a basic two-player Q learning algorithm for 2-TBSG. The algorithm is motivated by the two-player value iteration algorithm \cite{hansen2013strategy}. It is also motivated by the parametric Q-learning algorithm for solving MDP given by \cite{yang2019sample}.

\subsection{Algorithm and Parametrization}
\par The algorithm uses a vector $w\in\mathbb{R}^{K}$ to parametrize Q-functions, value functions and strategies as follows:
\begin{equation}\label{value1}
	\begin{aligned}
		& Q_{w}(s, a) = r(s, a) + \gamma \phi(s, a)^{T}w,\\
		& V_{w}(s) = \begin{cases}\max_{a\in\mathcal{A}} Q_{w}(s, a)\quad s\in\mathcal{S}_{1},\\\min_{a\in\mathcal{A}} Q_{w}(s, a)\quad s\in\mathcal{S}_{2},\end{cases}\quad
		 \pi_{w}(s) = \begin{cases}\arg\max_{a\in\mathcal{A}} Q_{w}(s, a)\quad s\in\mathcal{S}_{1},\\\arg\min_{a\in\mathcal{A}} Q_{w}(s, a)\quad s\in\mathcal{S}_{2}.\end{cases}
	\end{aligned}
\end{equation}
The algorithm keeps tracks of the parameter vector $w$ only. The value functions and strategies can be obtained from $w$ according to preceding equations when they are needed.

\par We present Algorithm \ref{alg1}, which is an approximate value iteration. Our algorithm picks a set $\mK$ of representative state-action pairs at first. Then at iteration $t$, it uses sampling to estimate the values $P(\cdot|s, a)^{T}V_{w^{(t-1)}}$, and carries value iteration using these estimates. The set $\mK$ can be chosen nearly arbitrarily, but it is necessary that the set is representative of the feature space. It means that the feature vectors of state-action pairs in this set cannot too be alike but need to be linearly independent. 
%That is, we require the following Assumption \ref{ass1}.
\begin{assumption}\label{ass1} There exist $K$ state-action pairs $(s, a)$ forming a set $\mathcal{K}$ satisfying
\begin{equation*}
	\|\phi(s, a)\|_{1}\le 1, \quad \forall s\in\mS,\ a\in\mA_{s},\qquad \exists L > 0,\quad \|\Phi_{\mathcal{K}}^{-1}\|_{\infty}\le L,
\end{equation*}
where $\Phi_{\mathcal{K}}$ is the $K\times K$ matrix formed by row features of those $(s, a)$ in $\mathcal{K}$.
\end{assumption}

{\small
\begin{algorithm}[htb!]\small
	\caption{Sampled Value Iteration for 2-TBSG}
	\label{alg1}
	\begin{algorithmic}[1]
		\State \textbf{Input: } A 2-TBSG $\mM = (\mS, \mA, P, r, \gamma)$, where $\mS = \mS_{1}\cup \mS_{2}, \mA = \mA_{1}\cup \mA_{2}$.
		\State \textbf{Input: }$\epsilon, \delta\in (0, 1)$.
		\State \textbf{Initialize: } $w^{(0)}\leftarrow 0$.
		\State \textbf{Initialize: } $R\leftarrow \Theta(1/(1-\gamma)\log(1/(1-\gamma)\epsilon)),\quad T\leftarrow \Theta\left[\left(L^{2}\cdot\log(KR/\delta)/(\epsilon^{2}(1-\gamma)^{6})\right)\right]$.
		\State Pick a set $\mK$ of state-action pairs, which satisfies Assumption \ref{ass1}. 
		\For{$t = 1:R$}
			\State $M^{(t)}\leftarrow \mathbf{0}\in \mathbb{R}^{K}$
			\For{$(s, a)\in\mK$}
				\State Sample $P(\cdot|s, a)$ for $T$ times to obtain $s_{1}^{(t)}, \cdots, s_{T}^{(t)}\in\mathcal{S}$.
				\State Let $M^{(t)}(s, a) = \frac{1}{T}\sum_{i=1}^{T}\Pi_{[0, 1/(1-\gamma)]}V_{w^{(t-1)}}(s_{i}^{(t)})$, where $V_{w}$ is defined in \eqref{value1}.
			\EndFor
			\State $w^{(t)}\leftarrow \Phi_{\mathcal{K}}^{-1}M^{(t)}$.
		\EndFor
		\State\textbf{Output:} $w^{(R)}\in \mathbb{R}^{K}$.
	\end{algorithmic}
\end{algorithm}
}

\subsection{Sample Complexity Analysis}

\par The next theorem establishes the sample complexity of Algorithm \ref{alg1}, which is independent from $|\mathcal{S}|$ and $|\mathcal{A}|$. Its proof is deferred to the appendix.
\begin{theorem}[Convergence of Algorithm \ref{alg1}]\label{thm1}
	Let Assumption \ref{ass1} holds. Suppose that the transition model $P$ of $\mM$ can be fully embedded into the $\phi$ space. Then for some $\epsilon > 0, \delta\in (0, 1)$, with probability at least $1-\delta$, the parametrized strategy $\pi_{w^{(R)}}$ according to the output of Algorithm \ref{alg1} is $\epsilon$-optimal. The number of samples used is $\tilde{\mathcal{O}}\left(\frac{KL^{2}}{(1-\gamma)^{7}\epsilon^{2}}\cdot \text{poly}\log(K\epsilon^{-1}\delta^{-1})\right)$.
\end{theorem}

\section{Variance-Reduced Q-Learning for Two-Player Stochastic Games}\label{sec5}
\par In this section, we show how to accelerate the two-player Q-learning algorithm and achieve near-optimal sample efficiency. A main technique is to leverage monotonicity of the Bellman operator to guarantee that solutions improve monotonically in the algorithm, which was used in \cite{yang2019sample}.
%several techniques to improve the sample complexity of calculating an $\epsilon$-optimal strategy,

\subsection{Nonnegative Features}
\par To preserve monotonicity in the algorithm, we assume without loss of generality the features are nonnegative:
	\begin{equation*}
	\phi_{k}(s, a)\ge 0, \quad \forall k\in [K],\quad \forall (s, a)\in\mS\times\mA
	\end{equation*}
This condition can be easily satisfied. %weaker than the anchor assumption in \cite{yang2019sample}. 
If the raw features $\phi$ does not satisfy nonnegativity, we can construct new features $\phi'$ to make it satisfied. For any state-action pair $(s, a)$ we append another 1D feature $\phi_{K+1}'(s, a)$ such that $\phi_{k}'(s, a) = \phi_{k}(s, a) + \phi_{K+1}'(s, a)\ge 0$ for $k\in[K]$, and there is a subset $\mK'$ of $\mS\times\mA$ such that $|\mK'| = K+1$ and $\Phi'_{\mK'}$ is nonsingular. Then $\phi'$ satisfies nonnegativity condition and Assumption \ref{ass1} for some $L$ by normalization. More details are deferred to appendix. 

\subsection{Parametrization}
\par We use a ``max-linear" parameterization to guarantee that value functions improve monotonically in the algorithm. Instead of using a single vector $w$, we now use a finite collection of $K$-dimensional vectors $\theta = \{w^{(h)}\}_{h=1}^{Z}$, where $Z$ is an integer of order $1/(1-\gamma)$. We use the following parameterization for the Q-functions, the value functions and strategies\footnote{Here $\max_{h\in [Z]}[\arg\min_{a\in\mA_{s}}Q_{w^{(h)}}(s, a)]$ is defined to be the solution of $a$ in the max-min problem: $\max_{h\in [Z]}\min_{a\in\mA_{s}}Q_{w^{(h)}}(s, a)$. The definition of $\max_{h\in [Z]}[\arg\max_{a\in\mA_{s}}Q_{w^{(h)}}(s, a)]$ is similar.}:
\begin{equation}\label{para}
	\begin{aligned}
		& Q_{w^{(h)}}(s, a) = r(s, a) + \gamma \phi(s, a)^{T}w^{(h)},\\
		& V_{\theta}(s) = \begin{cases}\max_{h\in[Z]}\max_{a\in\mA_{s}}Q_{w^{(h)}}(s, a), \quad \forall s\in\mathcal{S}_{1},\\\max_{h\in[Z]}\min_{a\in\mA_{s}}Q_{w^{(h)}}(s, a), \quad \forall s\in\mathcal{S}_{2},\end{cases}\\
		& \pi_{\theta}(s) = \begin{cases}\max_{h\in[Z]}\left[\arg\max_{a\in\mA_{s}}\left(Q_{w^{(h)}}(s, a)\right)\right], \quad \forall s\in\mathcal{S}_{1},\\\max_{h\in[Z]}\left[\arg\min_{a\in\mA_{s}}\left(Q_{w^{(h)}}(s, a)\right)\right], \quad \forall s\in\mathcal{S}_{2}.\end{cases}
	\end{aligned}
\end{equation}
For a given $\theta$ and $s$, computing the corresponding Q-value and action requires solving a one-step optimization problem. We assume that there is an oracle that solves the problem with time complexity $M$.
\begin{remark}
	When the action space is continuous, this $M$ may become a constant which is independent to the state set and the action set.
\end{remark}
%Here the second line of $\pi_{\theta}$'s definition is given by the solution of $a$ in the following optimization problem:
%\begin{equation*}
%	\max_{h\in[Z]}\min_{a\in\mA_{s}}\left(r(s, a) + \gamma \phi(s, a)^{T}w^{(h)}\right).
%\end{equation*}

\subsection{Preserving Monotonicity}

\par A drawback of value iteration-like method is: an $\epsilon$-optimal value function cannot be used greedily to obtain an $\epsilon$-optimal strategy. In order for Algorithm \ref{alg1} to find an $\pi$ such that $V^{\pi}$ is an $\epsilon$-optimal value, it needs to find an $(1-\gamma)\epsilon$-optimal value function first, which is very inefficient. However, if a strategy $\pi$ and a value function $V$ satisfy following inequality
\vspace{-2mm}
\begin{equation}\label{monotonic}
	V\le \mT_{\pi}V,
\end{equation}
then there is a strong connection between $V^{\pi}$ and $V$ as follows (due to monotonicity of the Bellman operator $\mT$):
\vspace{-2mm}
\begin{equation*}
	V\le \mT_{\pi}V\le \mT_{\pi}^{2}V\le \cdots\le \mT_{\pi}^{\infty}V = V^{\pi}.
\end{equation*}
This relation will be used to show that if $V$ is close to optimal, the policy $\pi$ is also close to optimal.

The accelerated algorithm is given partly in Algorithm \ref{alg2}, which uses two tricks to preserve monotonicity:
	\vspace{-2mm}
\begin{itemize}
	\item We use parametrization \eqref{para} for $Q_{w}, V_{\theta}$ and $\pi_{\theta}$. This parametrization ensures that in our algorithm, the values and strategies keeps improving throughout iterations.
	\vspace{-1mm}
	\item In each iteration, we shift downwards the new parameter $w^{(i, j)}$ to $\overline{w}^{(i, j)}$ by using a confidence bound, such that
	\begin{equation*}
		\phi(s, a)^{T}\overline{w}^{(i, j)}\le P(\cdot|s, a)^{T}V^{(i, j-1)}\le P(\cdot|s, a)^{T}V^{(i, j)},
	\end{equation*}
	which uses the nonnegativity of features. The shift is used to guarantee \eqref{monotonic}.
	\vspace{-1mm}
\end{itemize}

\subsection{Approximating the Equilibrium from Two Sides}
\par Making value functions monotonically increasing is not enough to find an $\epsilon$-optimal strategy for two-player stochastic games. There are two sides of the game, and $V^{\pi}$ may be either greater or less than the Nash value. Having a lowerbound $V$ for $V^{\pi}$ does not lead to an approximate strategy. This is a major difference from one-player MDP.

In order to fix this problem, we approximate the Nash equilibrium from two sides -- both from above and below. Given player 1's strategy $\pi_{1}$ and player 2's strategy $\pi_{2}$, we introduce two Bellman operators $\mT_{\pi_{1}, \min}, \mT_{\max, \pi_{2}}$ . 
\begin{equation}
	\begin{aligned}
		\mT_{\pi_{1}, \min}V & = \begin{cases}r(s, \pi_{1}(s)) + \gamma P(\cdot|s, \pi_{1}(s))^{T}V, & \quad \text{if }s\in\mS_{1},\\
		\min_{s\in\mA_{s}}\left[r(s, a) + \gamma P(\cdot|s, a)^{T}V\right], &\quad \text{if }s\in\mS_{2},\end{cases}\\
		\mT_{\max, \pi_{2}}V & = \begin{cases}\max_{s\in\mA_{s}}\left[r(s, a) + \gamma P(\cdot|s, a)^{T}V\right], & \quad \text{if }s\in\mS_{1},\\
		r(s, \pi_{2}(s)) + \gamma P(\cdot|s, \pi_{2}(s))^{T}V, &\quad \text{if }s\in\mS_{2}.\end{cases}
	\end{aligned}
\end{equation}
Then if there exist value functions $V, W$ such that all of the following 
\begin{equation}\label{ineq4}
	V\le \mT_{\pi_{1}, \min}V,\quad \mT_{\max, \pi_{2}}W\le W
\end{equation}
hold, then by using the monotonicity of $\mT_{\pi_{1}, \min}, \mT_{\pi_{2}, \max}$ we get
\begin{equation*}
	V\le \min_{\overline{\pi}_{2}}V^{\pi_{1}, \overline{\pi}_{2}}\le v^{*}\le \max_{\overline{\pi}_{1}}V^{\overline{\pi}_{1}, \pi_{2}}\le W.
\end{equation*}
Hence if we have $\|V - v^{*}\|_{\infty}\le \epsilon$ and $\|W - v^{*}\|_{\infty}\le \epsilon$, they jointly imply 
\begin{equation}
	\begin{aligned}
		& \|\min_{\overline{\pi}_{2}}V^{\pi_{1}, \overline{\pi}_{2}} - v^{*}\|_{\infty}\le \epsilon,\\
		& \|\max_{\overline{\pi}_{1}}V^{\overline{\pi}_{1}, \pi_{2}} - v^{*}\|_{\infty}\le \epsilon,
	\end{aligned}
\end{equation}
which indicates that $(\pi_{1}, \pi_{2})$ is an $\epsilon$-optimal strategy.
\par To achieve this goal, we construct a ``flipped" instance of 2-TBSG $\mM' = (\mS, \mA, P, r', \gamma)$, where the state set and the action set for each player, the transition probability matrix and the discounted factor are identical with those of $\mM$. The reward function $r'$ is defined as
\begin{equation}
	r'(s, a) = 1 - r(s, a).
\end{equation}
And the objective of two players are switched, which means in $\mM'$ the first player aims to minimize and the second player aims to maximize. $\mM,\mM'$ share the same optimal strategy (but flipped). 
\par We use $V'$ to denote the value function of $\mM'$, and let $W(s) = \frac{1}{1-\gamma} - V'(s)$ for any $s\in\mS$, which serves as the value function approximating the equilibrium value $v^{*}$ from upper side.  This $W$, together with $V$, forms a two-sided approximation to the equilibrium value function.
\par We use Algorithm \ref{alg2} to solve $\mM$ and $\mM'$ at the same time. Next we construct a strategy $\pi$ where the first player's strategy is based on parameters from the lower approximation, and the second player's strategy is based on parameters from the upper approximation. This process is described in Algorithm \ref{alg3}, and its output is the following approximate Nash equilibrium strategy: 
\begin{equation}
	\pi(s) = \begin{cases}
		\pi_{\theta^{(R', R)}}(s), \quad\text{if }s\in\mS_{1},\\
		\pi_{\eta^{(R', R)}}'(s), \quad\text{if }s\in\mS_{2},
	\end{cases}
\end{equation}
where for $\eta = \{z^{(h)}\}_{h = 1}^{Z}$, $\pi'$ is the strategy defined as
\begin{equation*}
	\pi_{\eta}'(s) = \begin{cases}\max_{h\in[Z]}\left[\arg\min_{a\in\mA_{s}}\left(r'(s, a) + \gamma \phi(s, a)^{T}z^{(h)}\right)\right], \quad \forall s\in\mathcal{S}_{1},\\\max_{h\in[Z]}\left[\arg\max_{a\in\mA_{s}}\left(r'(s, a) + \gamma \phi(s, a)^{T}z^{(h)}\right)\right], \quad \forall s\in\mathcal{S}_{2},\end{cases}
\end{equation*}

\subsection{Variance Reduction}
\par We use inner-outer loops for variance reduction in Algorithm \ref{alg2}. Let the parameters at the $(i,j)$-th iteration be $\theta^{(i, j)}$. At the beginning of the $i$-th outer iteration, we aim to approximate $P(\cdot|s, a)^{T}V_{\theta^{(i, 0)}}$ accurately (Step 6, 7). Then in the $(i,j)$-th inner iteration, we use $P(\cdot|s, a)^{T}V_{\theta^{(i, 0)}}$ as a reference to reduce the variance of estimation. That is, we estimate the difference $P(\cdot|s, a)^{T}(V_{\theta^{(i, j-1)}} - V_{\theta^{(i, 0)}})$ using samples and then use the following equation (Step 11, 12)
\begin{equation*}
	P(\cdot|s, a)^{T}V_{\theta^{(i, j-1)}} = P(\cdot|s, a)^{T}V_{\theta^{(i, 0)}} + P(\cdot|s, a)^{T}(V_{\theta^{(i, j-1)}} - V_{\theta^{(i, 0)}})
\end{equation*}
to approximate $P(\cdot|s, a)^{T}V_{\theta^{(i, j-1)}}$. Since the infinite norm of $(V_{\theta^{(i, j-1)}} - V_{\theta^{(i, 0)}})$ is guaranteed to be smaller than the absolute value of $V_{\theta^{(i, 0)}}$, the number of samples needed for each inner iteration can be substantially reduced. 
Hence our algorithm is more sample-efficient.

\subsection{Putting Together}
\par Algorithms \ref{alg2}-\ref{alg3} puts together all the techniques that were explained. In the next section, we will prove that they output an $\epsilon$-optimal strategy with high probability.
It is easy to see that the time complexity of Algorithm \ref{alg2} is $\tilde{\mathcal{O}}(K^{\omega} + MK / ((1-\gamma)^{4}\epsilon^{2}) + K^{2} / (1-\gamma))$. The first term $K^{\omega}$ is the time calculating $\Phi_{\mK}^{-1}$. The second term is the time of sampling and calculating the value function in each iteration, and $M$ is the time of calculating $V_{\theta}(s)$ given parameter $\theta$ and state $s$, which can be viewed as solving an optimization problem over the action space. The last term is due to the calculation of $\Phi_{\mK}M^{(i, j)}$. As for the space complexity, we only need to store $\Phi_{\mK}^{-1}$ and the parameter $\theta^{(i, j)}$ at each iteration, which take $\tilde{\mathcal{O}}(K^{2} + K/(1-\gamma))$ space. %In Algorithm \ref{alg3}, note that in order to execute Algorithm \ref{alg2} for the 2-TBSG $\mM'$, we only need to calculate $1 - r(s, a)$ for those $(s, a)$ encountered in calculating $V$, which will be counted into the $M$. 
Hence the total time and space complexities are independent from the numbers of states and actions.

{\small
\begin{algorithm}[htb!]\small
	\caption{One-side Parametric Q-Learning with Variance Reduction for 2-TBSG}
	\label{alg2}
	\begin{algorithmic}[1]
		\State \textbf{Input: } A two-TBSG $\mM = (\mS, \mA, P, r, \gamma)$ with feature map $\phi$, where $\mS = \mS_{1}\cup \mS_{2}, \mA = \mA_{1}\cup \mA_{2}$
		\State \textbf{Input: }$\epsilon, \delta\in (0, 1)$
		\State \textbf{Initialize: } 
		\begin{equation*}
			\begin{aligned}
				& R'\leftarrow\Theta(\log 1/(\epsilon(1-\gamma))),\quad R\leftarrow \Theta(R'/(1-\gamma)),\quad \theta^{(0, 0)}\leftarrow\{\mathbf{0}\}\in\mathbb{R}^{K}\\
				& m\leftarrow\Theta(L^{2}\log(R'RK\delta^{-1})/(\epsilon^{2}(1-\gamma)^{4})),\quad m_{1}\leftarrow\Theta(L^{2}\log(R'RK\delta^{-1})/((1-\gamma)^{2})),\\
				& \epsilon_{1}\leftarrow\Theta[L/(1-\gamma)\cdot\sqrt{\log(RR'K\delta^{-1})/m}],\\
				& \epsilon^{(i)}\leftarrow\epsilon_{1} + \Theta[L\cdot 2^{-i}/(1-\gamma)\sqrt{\log(RR'K\delta^{-1})/m_{1}}], \quad \forall 0\le i\le R
			\end{aligned}
		\end{equation*}
		\For{$i = 0, 1, \cdots, R'$}
			\For{$k\in [K]$}
				\State Generate state samples $x_{k}^{(1)}, \cdots, x_{k}^{(m)}\in\mathcal{S}$ from $P(\cdot|s_{k}, a_{k})$ for $(s_{k}, a_{k})\in\mK$. 
				\State Let $M^{(i, 0)}(k)\leftarrow \frac{1}{m}\sum_{l=1}^{m}V_{\theta^{(i, 0)}}(x_{k}^{(l)})$
			\EndFor
			\For{$j = 1, \cdots, R$}
			\For{$k\in [K]$}
				\State Generate state samples $x_{k}^{(1)}, \cdots, x_{k}^{(m_{1})}\in\mathcal{S}$ from $P(\cdot|s_{k}, a_{k})$ for $(s_{k}, a_{k})\in\mK$. 
				\State Let $M^{(i, j)}(k)\leftarrow\frac{1}{m_{1}}\sum_{l=1}^{m_{1}}\left(V_{\theta^{(i, j-1)}}(x_{k}^{(l)}) - V_{\theta^{(i, 0)}}(x_{k}^{(l)})\right) + M^{(i, 0)}$
			\EndFor
			\State $w^{(i, j)}\leftarrow \Phi_{\mK}^{-1}M^{(i, j)}$
			\State $\overline{w}^{(i, j)}(k)\leftarrow w^{(i, j)}(k) - \epsilon^{(i)}, \quad \forall k\in [K]$
			\State $\theta^{(i, j)}\leftarrow \theta^{(i, j-1)}\cup\{\overline{w}^{(i, j)}\}$
			\EndFor
			\State $\theta^{(i+1, 0)}\leftarrow \theta^{(i, R)}$
		\EndFor
		\State\textbf{Output: }$\theta^{(R', R)}$
	\end{algorithmic}
\end{algorithm}
}

{\small
\begin{algorithm}[htb!]\small
	\caption{Two-side Parametric Q-Learning with Variance Reduction}
	\label{alg3}
	\begin{algorithmic}[1]
		\State \textbf{Input: } A two-TBSG $\mM = (\mS, \mA, P, r, \gamma)$ with feature map $\phi$, where $\mS = \mS_{1}\cup \mS_{2}, \mA = \mA_{1}\cup \mA_{2}$
		\State \textbf{Input: } $\epsilon, \delta\in (0, 1)$
		\State Construct $\mM' = (\mS, \mA, P, 1 - r, \gamma)$ where objectives of two players are switched
		\State Solve $\theta^{(R', R)}$ using Algorithm \ref{alg2} with input $\mM$ and $\epsilon, \delta$
		\State Solve $\eta^{(R', R)}$ using Algorithm \ref{alg2} with input $\mM'$ and $\epsilon, \delta$
		\State Construct a strategy $\pi: \mS\to\mA$:
		\begin{equation}
			\pi(s) = 
			\begin{cases}
				\pi_{\theta^{(R', R)}}(s), \quad \forall s\in\mS_{1},\\
				\pi_{\eta^{(R', R)}}'(s), \quad \forall s\in\mS_{2}.
			\end{cases}
		\end{equation}
		\State \textbf{Output: }$\pi$
	\end{algorithmic}
\end{algorithm}
}

\section{Sample Complexity of Algorithms \ref{alg2}-\ref{alg3}}\label{sec6}
\par In this section, we analyze the sample complexity of our Algorithms \ref{alg2}-\ref{alg3}.
\begin{theorem}\label{thm2}
	Let Assumption \ref{ass1} hold and let features be nonnegative. Suppose that the transition model $P$ of $\mM$ can be fully embedded into $\phi$ space. Then for some $\epsilon > 0, \delta\in (0, 1)$, with probability at least $1 - 2\delta$, the output of Algorithm \ref{alg3} is an $\epsilon$-optimal strategy. The number of samples used is $N = \mathcal{O}\left(\frac{KL^{2}}{(1-\gamma)^{4}\epsilon^{2}}\cdot\text{poly}\log(K\epsilon^{-1}\delta^{-1})\right)$.
\end{theorem}

We present a proof sketch here, and the complete proof is deferred to appendix.

\begin{proof}[Proof Sketch]
	\par We prove $\|V_{\theta^{(i, 0)}} - v^{*}\|_{\infty}\le 2^{-i}/(1-\gamma)$ by induction. It is easy to know that $\|V_{\theta^{(0, 0)}} - v^{*}\|_{\infty}\le 1/(1-\gamma)$. Next we assume $\|V_{\theta^{(i-1, 0)}} - v^{*}\|_{\infty}\le 2^{-i+1}/(1-\gamma)$ holds.
	\par The error between $V_{\theta^{(i, 0)}}$ and $v^{*}$ involves two types of error: the estimation error due to sampling and the convergence error of value iteration. Due to the variance reduction technique, estimation error has two parts. The first part is the estimation error of $\Phi_{\mK}^{-1}P_{\mK}V_{\theta^{(i-1, 0)}}$, which we denote as $\epsilon^{(i, 0)}$, and the second part is the error of $\Phi_{\mK}^{-1}P_{\mK}(V_{\theta^{(i-1, j)}} - V_{\theta^{(i-1, 0)}})$, which we denote as $\epsilon^{(i, j)}$ for short. 
	\par According to the Hoeffding inequality, we have $\epsilon^{(i, 0)} = \tilde{\mathcal{O}}(L/(1-\gamma)\cdot\sqrt{1/m})$ and $\epsilon^{(i, j)} = \tilde{\mathcal{O}}(L\cdot\max_{s}|V_{\theta^{(i-1, j-1)}}(s) - V_{\theta^{(i-1, 0)}}(s)|\cdot\sqrt{1/m_{1}})$ with high probability. By the induction hypothesis, we have $\max_{s}|V_{\theta^{(i-1, j-1)}}(s) - V_{\theta^{(i-1, 0)}}(s)|\le \frac{1}{2^{i-1}(1-\gamma)}$. If we choose $m = c L^{2}/((1-\gamma)^{4}\epsilon^{2})$ and $m_{1} = c_{1} L^{2}/(1-\gamma)^{2}$, we will have $\epsilon^{(i, 0)} = \mathcal{O}(\epsilon(1-\gamma))$ and $\epsilon^{(i, j)} = \mathcal{O}(2^{-i})$. 
	\par The convergence error of value iteration in the inner loop is $\gamma^{R}/(1-\gamma)$. If we choose $R = c_{R}\log(\epsilon^{-1}(1-\gamma)^{-1})/(1-\gamma)$, we will have $\gamma^{R}/(1-\gamma) = \mathcal{O}(\epsilon)$. Bringing these two types of errors together, we have with high probability that
	\begin{equation*}
		\begin{aligned}
			\|v^{*} - V_{\theta^{(i, 0)}}\|_{\infty} & \le \gamma^{R}/(1-\gamma) + \sum_{j=1}^{R}\gamma^{R-j}\cdot(\epsilon^{(i, 0)} + \epsilon^{(i, j)})\\
			& \le \mathcal{O}\left(\epsilon + (\epsilon(1-\gamma) + 2^{-i})/(1-\gamma)\right) = \mathcal{O}\left(2^{-i}/(1-\gamma)\right),
		\end{aligned}
	\end{equation*}
	where the last equality is due to $\epsilon \le 2^{-R'} / (1-\gamma)\le 2^{-i}/(1-\gamma)$. Choosing $R' = c_{R'}\log(\epsilon^{-1}(1-\gamma)^{-1})$, we have $\|v^{*} - V_{\theta^{(R', R)}}\|_{\infty} = \|v^{*} - V_{\theta^{(R+1, 0)}}\|_{\infty}\le \epsilon$. 
	Here we have omitted the dependence on any constant factors.%do not take the influence of constant into consideration, since we can always adjust $c, c_{1}, c_{R}$ to get our desired constant in the above estimation. This completes the induction.
	\par Similarly, we can show $\|v^{*} - W_{\eta^{(R+1, 0)}}\|_{\infty}\le \epsilon$ for the ``flipped" side. Hence we have $V_{\theta^{(R', R)}}\le v^{*}\le W_{\eta^{(R', R)}}$,  therefore the combined strategy $\pi = (\pi_{1}, \pi_{2})$ given by Algorithm \ref{alg3} is an $\epsilon$-optimal strategy since
	\begin{equation}
		V_{\theta^{(R', R)}}\le \min_{\overline{\pi}_{2}}V^{\pi_{1}, \overline{\pi}_{2}}\le v^{*}\le \max_{\overline{\pi}_{1}}V^{\overline{\pi}_{1}, \pi_{2}}\le W_{\eta^{(R, R')}},
	\end{equation}
whose proof is based on the monotonicity of two operators $\mT_{\pi_{1}, \min}, \mT_{\max, \pi_{2}}$. The total number of samples used by Algorithm \ref{alg3} is $R'(R\cdot m_{1} + m)  = \tilde{\mathcal{O}}(L^{2} / ((1-\gamma)^{4}\epsilon^{2}))$.
\end{proof}
\par According to Theorem \ref{thm1} and \ref{thm2}, we have the following theorem when the transition model cannot be embedded exactly, whose proof is deferred to appendix.
\begin{theorem}[Approximation error due to model misspecification]\label{thm4}
	Let Assumption \ref{ass1} holds and let features be nonnegative. If there is an another transition model $\tilde{P}$ which can be fully embedded into $\phi$ space, and there exists $\xi\in[0, 1]$ such that $\|P(\cdot|s, a) - \tilde{P}(\cdot|s, a)\|_{TV}\le \xi$ for $\forall (s, a)\in\mS\times\mA$ and $P(\cdot|s, a) = \tilde{P}(\cdot|s, a)$ for $(s, a)\in\mK$, then with probability at least $1 - 2\delta$, the output of Algorithm \ref{alg3} is an $\left(2\xi/(1-\gamma)^{2} + 2\epsilon\right)$-optimal strategy, and with probability at least $1 - \delta$, the parametrized strategy $(\pi_{w^{(R)}}^{1}, \pi_{w^{(R)}}^{2})$ according to the output of Algorithm \ref{alg1} is $\left(2\xi/(1-\gamma)^{2} + 2\epsilon\right)$-optimal.
\end{theorem}

\section{Conclusion}\label{sec7}
\par In this paper, we develop a two-player Q-learning algorithm for solving 2-TBSG in feature space. This algorithm is proved to find an $\epsilon$-optimal strategy with high probability using $\tilde{O}(K/((1-\gamma)^{4}\epsilon^{2}))$ samples. %Though there is still a $1/(1-\gamma)$ gap between the lower bound of sample complexity and the complexity of our algorithm, 
It is the first and sharpest sample complexity bound for solving two-player stochastic game using features and linear models, to our best knowledges. The algorithm is sample efficient as well as space and time efficient.

%\setcitestyle{numbers}

\bibliography{reference}

\newpage
\appendix

\section{Proof of Theorem \ref{thm1}}
\par We first present the definition of optimal counterstrategies.
\begin{definition}
	For player 1's strategy $\pi_{1}$, we call $\pi_{2}$ a player 2's optimal counterstrategy against $\pi_{1}$, if for any player 2's strategy $\overline{\pi}_{2}$, we have $V^{\pi_{1}, \pi_{2}}\le V^{\pi_{1}, \overline{\pi}_{2}}$. For player 2's strategy $\pi_{2}$, we call $\pi_{1}$ a player 1's optimal counterstrategy against $\pi_{2}$, if for any player 1's strategy $\overline{\pi}_{1}$, we have $V^{\pi_{1}, \pi_{2}}\ge V^{\overline{\pi}_{1}, \pi_{2}}$.
\end{definition}
It is known in \cite{puterman2014markov} that for any player 1's strategy $\pi_{1}$ (player 2's strategy $\pi_{2}$), the optimal counterstrategy against $\pi_{1}$ ($\pi_{2}$) always exists.
\par Our next lemma indicates that we can use the error of parametrized $Q$ functions to bounded the error of value functions of parametrized strategies.
\begin{lemma}\label{lem}
 	If 
	\begin{equation}\label{eq1}
		\|Q_{w}(s, a) - Q^{*}(s, a)\|_{\infty}\le \zeta,
	\end{equation}
	then we have
	\begin{equation}\label{ineq3}
		\|v^{\pi_{1}, \pi_{2}^{*}} -v^{*}\|_{\infty}\le \frac{2\zeta}{1-\gamma}, \quad \|v^{\pi_{1}^{*}, \pi_{2}} -v^{*}\|_{\infty}\le \frac{2\zeta}{1-\gamma},
	\end{equation}
	where $(\pi_{1}, \pi_{2}) = \pi_{w}$, and $\pi_{1}^{*}, \pi_{2}^{*}$ are optimal counterstrategies of $\pi_{2}, \pi_{1}$.
\end{lemma}
\begin{proof}
	We only prove the first inequality of \eqref{ineq3}. The proof of the second inequality is similar. 
	\par For any $s\in\mS_{1}$,
	\begin{equation*}
		\begin{aligned}
			|v^{\pi_{1}, \pi_{2}^{*}}(s) - v^{*}(s)| & = |v^{\pi_{1}, \pi_{2}^{*}}(s) - Q^{*}(s, \pi^{*}(s))|\\
			& \le |v^{\pi_{1}, \pi_{2}^{*}}(s) - Q^{*}(s, \pi_{1}(s))| + |Q^{*}(s, \pi_{1}(s)) - Q^{*}(s, \pi^{*}(s))|\\
			& = |\gamma P(\cdot|s, \pi_{1}(s))v^{\pi_{1}, \pi_{2}^{*}} - \gamma P(\cdot|s, \pi_{1}(s))v^{*}| + |Q^{*}(s, \pi_{1}(s)) - Q^{*}(s, \pi^{*}(s))|\\
			& \le \gamma|P(\cdot|s, \pi_{1}(s))(v^{\pi_{1}, \pi_{2}^{*}} - v^{*})| + |Q^{*}(s, \pi_{1}(s)) - Q^{*}(s, \pi^{*}(s))|\\
			& \le \gamma\|v^{\pi_{1}, \pi_{2}^{*}} - v^{*}\|_{\infty} + |Q^{*}(s, \pi_{1}(s)) - Q^{*}(s, \pi^{*}(s))|.
		\end{aligned}
	\end{equation*}
	According to the definition of $\pi_{w} = (\pi_{1}, \pi_{2})$ and $Q^{*}$,
	\begin{equation*}
		\begin{aligned}
			& Q^{*}(s, \pi_{1}(s))\le Q^{*}(s, \pi^{*}(s)),\\
			& Q_{w}(s, \pi_{1}(s))\ge Q_{w}(s, \pi^{*}(s)).
		\end{aligned}
	\end{equation*}
	Combine this inequality with inequality \eqref{eq1}, we get
	\begin{equation*}
		\begin{aligned}
			0\le &\ Q^{*}(s, \pi^{*}(s)) - Q^{*}(s, \pi_{1}(s))\\
			\le &\ |Q^{*}(s, \pi^{*}(s)) - Q_{w}(s, \pi^{*}(s))| + Q_{w}(s, \pi^{*}(s)) - Q_{w}(s, \pi_{1}(s))\\
			&\ + |Q_{w}(s, \pi_{1}(s)) - Q^{*}(s, \pi_{1}(s))|\\
			\le &\ 2\|Q_{w} - Q^{*}\|_{\infty}\le 2\zeta,
		\end{aligned}
	\end{equation*}
	which indicates that
	\begin{equation*}
		|v^{\pi_{1}, \pi_{2}^{*}}(s) - v^{*}(s)|\le \gamma\|v^{\pi_{1}, \pi_{2}^{*}} - v^{*}\|_{\infty} + 2\zeta, \quad \forall s\in\mS_{1}.
	\end{equation*}
	Furthermore, for $s\in\mS_{2}$, we have
	\begin{equation*}
		\begin{aligned}
			|v^{\pi_{1}, \pi_{2}^{*}}(s) - v^{*}(s)| & = \min_{a\in\mA_{s}}\left[r(s, a) + \gamma P(\cdot|s, a)^{T}V^{\pi_{1}, \pi_{2}^{*}}\right] - \min_{a\in\mA_{s}}\left[r(s, a) + \gamma P(\cdot|s, a)^{T}V^{*}\right]\\
			& \le \max_{a\in\mA_{s}}\left[\gamma P(\cdot|s, a)^{T}V^{\pi_{1}, \pi_{2}^{*}} - \gamma P(\cdot|s, a)^{T}V^{*}\right]\\
			& \le \gamma \|v^{\pi_{1}, \pi_{2}^{*}} - v^{*}\|_{\infty}.
		\end{aligned}
	\end{equation*}
	\par Therefore, we have
	\begin{equation*}
		\|v^{\pi_{1}, \pi_{2}^{*}} - v^{*}\|_{\infty}\le\gamma\|v^{\pi_{1}, \pi_{2}^{*}} - v^{*}\|_{\infty} + 2\zeta,
	\end{equation*}
	which indicates that
	\begin{equation*}
		\|v^{\pi_{1}, \pi_{2}^{*}} - v^{*}\|_{\infty}\le \frac{2\zeta}{1-\gamma}.
	\end{equation*}
	The first inequality of \eqref{ineq3} is verified.
\end{proof}

\begin{proof}[Proof of Theorem \ref{thm1}]
\par We define
\begin{equation*}
	\begin{aligned}
	\mathcal{T}V(s) = &\begin{cases}
		\max_{a\in\mathcal{A}_{s}}r(s, a) + \gamma P_{s, a}^{T}V, \quad \forall s\in\mathcal{S}_{1},\\
		\min_{a\in\mathcal{A}_{s}}r(s, a) + \gamma P_{s, a}^{T}V, \quad \forall s\in\mathcal{S}_{2},\\
	\end{cases}\\
	\hat{\mathcal{T}}^{(t)}V(s) = &\begin{cases}
		\max_{a\in\mathcal{A}_{s}}r(s, a) + \gamma\phi(s, a)^{T}\Phi_{\mathcal{K}}^{-1}\hat{P}_{\mathcal{K}}^{(t)}V, \quad \forall s\in\mathcal{S}_{1},\\
		\min_{a\in\mathcal{A}_{s}}r(s, a) + \gamma\phi(s, a)^{T}\Phi_{\mathcal{K}}^{-1}\hat{P}_{\mathcal{K}}^{(t)}V, \quad \forall s\in\mathcal{S}_{2},\\
	\end{cases}\\
	\end{aligned}
\end{equation*}
where $\hat{P}^{(t)}\in\mathbb{R}^{K\times|\mS|}$ ($1\le t\le R$) is the approximate transition probability matrix obtained by sampling at $t$-th iterations:
\begin{equation*}
	\hat{P}_{(s, a), s'}^{(t)} = \frac{1}{T}\sum_{i=1}^{T}1_{[s_{i}^{(t)} = s']}, \qquad \forall (s, a)\in\mathcal{K},\quad s_{1}^{(t)}, \cdots, s_{T}^{(t)}\sim P(\cdot|s, a).
\end{equation*}
Then our algorithm can be written as
\begin{equation*}
	V_{w^{(t)}}\leftarrow \hat{\mathcal{T}}^{(t)}\hat{V}_{w^{(t-1)}},\quad\forall 1\le l\le R,
\end{equation*}
where $\hat{V}_{w^{(t-1)}} = \Pi_{[0, 1/(1-\gamma)]}V_{w^{(t-1)}}$.

\par We define the following event as $\mE_{t}$:
\begin{equation*}
	\|(\hat{P}_{\mathcal{K}}^{(t)} - P_{\mathcal{K}})\hat{V}_{w^{t-1}}\|_{\infty}\le \epsilon_{1} := c\cdot \frac{1}{1-\gamma}\sqrt{\frac{\log(KR / \delta)}{T}}.
\end{equation*}
According to Hoeffding inequality for both state-action pairs in $\mK$ and applying their union bound, the event $\mE_{t}$ holds with probability at least $1 - \delta / R$. Also $\mE_{t}$ indicates that for any $s\in\mS, a\in\mA_{s}$, we have
\begin{equation}\label{eq2}
	\left|\gamma\phi(s, a)^{T}\Phi_{\mK}^{-1}\hat{P}_{\mK}^{(t)}\hat{V}_{w^{(t-1)}} - \gamma P_{s, a}^{T}\hat{V}_{w^{(t-1)}}\right|\le \ \left|\phi(s, a)^{T}\Phi_{\mK}^{-1}(\hat{P}_{\mK}^{(t)} - P_{\mK})\hat{V}_{w^{(t-1)}}\right|\le L\epsilon_{1},
\end{equation}
where we let $P_{s, a} = P(\cdot|s, a)$ and the last inequality comes from the assumption $\|\phi(s, a)\|_{1}\le 1$ and $\|\Phi_{\mK}^{-1}\|_{\infty}\le L$.
Noting that for any $s\in\mS$,
\begin{equation*}
	|[\hat{\mT}^{(t)}\hat{V}_{w^{(t-1)}}](s) - [\mT\hat{V}_{w^{(t-1)}}](s)| \le \max_{a\in\mA_{s}}\left|\gamma\phi(s, a)^{T}\Phi_{\mK}^{-1}\hat{P}_{\mK}^{(t)}\hat{V}_{w^{(t-1)}} - \gamma P_{s, a}^{T}\hat{V}_{w^{(t-1)}}\right|,
\end{equation*}
and $\mathcal{T}v^{*} = v^{*}$, we have
\begin{equation*}
	\begin{aligned}
		\|V_{w^{(t)}} - v^{*}\|_{\infty} & \le \|\hat{\mT}^{(t)}\hat{V}_{w^{(t-1)}} - \mT \hat{V}_{w^{(t-1)}}\|_{\infty} + \|\mT \hat{V}_{w^{(t-1)}} - \mT v^{*}\|_{\infty}\\
		& \le L\epsilon_{1} + \|\mT \hat{V}_{w^{(t-1)}} - \mT v^{*}\|_{\infty}\\
		& \le L\epsilon_{1} + \gamma\|\hat{V}_{w^{(t-1)}} - v^{*}\|_{\infty},
	\end{aligned}
\end{equation*}
where in the last inequality we use the contraction property of $\mT$.
\par Therefore, when $\mE^{(1)}, \cdots, \mE^{(R)}$ all hold, we get
\begin{equation*}
	\begin{aligned}
		\|V_{w^{(R)}} - V^{*}\|_{\infty} & \le L\epsilon_{1} + \gamma\|\hat{V}_{w^{(R-1)}} - v^{*}\|_{\infty}\\
		& \le L\epsilon_{1} + \gamma\|V_{w^{(R-1)}} - v^{*}\|_{\infty}\\
		& \le \gamma\cdot L\epsilon_{1} + L\epsilon_{1} + \gamma^{2}\|\hat{V}_{w^{(R-2)}} - v^{*}\|_{\infty}\\
		& \le \cdots\\
		& \le \frac{\gamma}{1-\gamma}\cdot L\epsilon_{1} + \frac{\gamma^{R}}{1-\gamma},\\
	\end{aligned}
\end{equation*}
where we use the fact $\|V_{w^{(0)}} - V^{*}\|_{\infty}\le 1 / (1-\gamma)$ in the last inequality. Furthermore, according to $Q_{w^{(R)}}(s, a) = r(s, a) + \gamma\phi(s, a)^{T}\Phi_{\mK}^{-1}\hat{P}_{\mK}^{(R)}\hat{V}_{w^{(R-1)}}$, we have
\begin{equation*}
	\|Q_{w^{(R)}} - Q^{*}\|_{\infty}\le L\epsilon_{1} + \gamma\|\hat{V}_{w^{(R-1)}} - v^{*}\|_{\infty}\le \frac{\gamma\cdot L\epsilon_{1}}{1-\gamma} + \frac{\gamma^{R}}{1-\gamma},
\end{equation*}
when $\mE^{(R)}$ holds, similar to the derivation in \eqref{eq2}. Finally if we suppose $\overline{\pi}_{w^{(R)}}^{1}, \overline{\pi}_{w^{(R)}}^{2}$ are optimal counterstrategies against $\pi_{w^{(R)}}^{2}, \pi_{w^{(R)}}^{1}$, using Lemma \ref{lem} we get
\begin{equation*}
	\begin{aligned}
		& \left\|V^{{\pi_{w^{(R)}}^{1}}, \overline{\pi}_{w^{(R)}}^{2}} - V^{*}\right\|_{\infty}\le \frac{2\gamma\cdot L\epsilon_{1}}{(1-\gamma)^{2}} + \frac{2\gamma^{R}}{(1-\gamma)^{2}}\\
		& \left\|V^{{\overline{\pi}_{w^{(R)}}^{1}}, \pi_{w^{(R)}}^{2}} - V^{*}\right\|_{\infty}\le \frac{2\gamma\cdot L\epsilon_{1}}{(1-\gamma)^{2}} + \frac{2\gamma^{R}}{(1-\gamma)^{2}}
	\end{aligned}
\end{equation*}
if events $\mE_{1}, \cdots, \mE_{R}$ all hold. Since every $\mE_{t}$ holds with probability at least $1 - \delta / R$, the probability when all these events hold is at least $1 - \delta$.
\par Hence if we choose 
\begin{equation*}
	R = \frac{1}{1-\gamma}\log\frac{1}{\epsilon(1-\gamma)}, \quad \epsilon_{1} = \mathcal{O}\left(\frac{\epsilon(1-\gamma)^{2}}{L}\right), \quad T = \mathcal{O}\left(\frac{\log(KR/\delta)}{\epsilon_{1}^{2}(1-\gamma)^{2}}\right) = \mathcal{O}\left(\frac{L^{2}\cdot \log(KR/\delta)}{\epsilon^{2}(1-\gamma)^{6}}\right),
\end{equation*}
we can obtain a strategy $\pi^{(R)}$ such that
\begin{equation*}
	\left\|V^{{\pi_{w^{(R)}}^{1}}, \overline{\pi}_{w^{(R)}}^{2}} - V^{*}\right\|_{\infty}\le \epsilon, \quad \left\|V^{{\overline{\pi}_{w^{(R)}}^{1}}, \pi_{w^{(R)}}^{2}} - V^{*}\right\|_{\infty}\le \epsilon
\end{equation*}
with probability at least $1 - \delta$. And the number of samples required is
\begin{equation*}
	T\cdot R\cdot k = \tilde{\mathcal{O}}\left(\frac{kL^{2}}{\epsilon^{2}(1-\gamma)^{7}}\right).
\end{equation*}
\end{proof}

\section{Construction of Non-negative Feature}
\par In this section, we discuss how to construct nonnegative features such that the Assumption \ref{ass1} holds.
\begin{theorem}\label{thm3}
	Suppose the transition model $P$ can be embedded into $\phi$. If there exists a set $\mK$ of state-action pair such that $|\mK| = K$ and $\Phi_{\mK}$ is nonsingular, then we can construct dimension-$(k+1)$ nonnegative features $\phi'(s, a)$ for every state action pairs $(s, a)$ such that $P$ can also be embedded into $\phi'$, and Assumption \ref{ass1} holds for a state-action pair set $\mK'$ and some constant $L$.
\end{theorem}
\par We first present a lemma which is useful in proving Theorem \ref{thm3}.
\begin{lemma}\label{lem4}
	If $A$ is a non-singular matrix and $k\in[K]$, and $A'(l)$ is the following matrix
	\begin{equation*}
		A'(l)_{i, j} = \begin{cases}
			A_{i, j} + l, \quad & \text{if }i = k;\\
			A_{i, j}, \quad & \text{if }i\neq k.
		\end{cases}
	\end{equation*}
	Then for any $N > 0$, there exists $l > N$ such that $A'(l)$ is non-singular.
\end{lemma}
\begin{proof}
	This lemma directly follows from the fact that $\det(A'(l))$ is a linear function of $l$.
\end{proof}
\begin{proof}[Proof of Theorem \ref{thm3}]
	We construct a feature map $\mathcal{L}$ from $\phi'(s, a)$ to $\phi(s, a)$ in the following way.
	\begin{equation}\label{eq4}
		\begin{aligned}
			&\ \mathcal{L}\left([\phi_{1}'(s, a), \phi_{2}'(s, a), \cdots, \phi_{K}'(s, a), \phi_{K+1}'(s, a)]\right)\\
			= &\ [\phi_{1}'(s, a) - \phi_{K+1}'(s, a), \phi_{2}'(s, a) - \phi_{K+1}'(s, a), \cdots, \phi_{K}'(s, a) - \phi_{K+1}'(s, a)]\\
			= &\ [\phi_{1}(s, a), \phi_{2}(s, a), \cdots, \phi_{K}(s, a)],
		\end{aligned}
	\end{equation}
	which means $\phi_{k}(s, a) = \phi_{k}'(s, a) - \phi_{K+1}'(s, a)$ for any $k\in[K]$. Adopting this feature map, $P$ can be embedded into $\phi'$. Hence we only need to construct nonnegative features $\phi'$ satisfying both \eqref{eq4} and Assumption \ref{ass1}.
	\par For any $(s, a)$ there exists $N(s, a)$ such that for any $\phi'_{K+1}(s, a)\ge N(s, a)$, we have $\phi_{k}(s, a) + \phi'_{K+1}(s, a)\ge 0$ for any $k\in[K]$. We choose arbitrarily a state-action pair $(s', a')$ not in $\mK$, and let $\mK'$ to be the union of $\mK$ and $\{(s', a')\}$. If we choose $\phi'_{K+1}(s', a') > \max\{0, N(s', a')\}$ and $\phi'_{k}(s', a') = \phi_{k}(s', a') + \phi'_{K+1}(s', a')$, we will have
	\begin{equation*}
		\det\begin{bmatrix}\phi_{1}(s_{1}, a_{1}) & \cdots & \phi_{K}(s_{1}, a_{1}) & 0\\\phi_{1}(s_{2}, a_{2}) & \cdots & \phi_{K}(s_{2}, a_{2}) & 0\\\vdots & \ddots & \vdots & \vdots\\\phi_{1}(s_{K}, a_{K}) & \cdots & \phi_{K}(s_{K}, a_{K}) & 0\\ \phi_{1}'(s', a') & \cdots & \phi_{K}'(s', a') & \phi_{K+1}'(s', a')\end{bmatrix}\neq 0
	\end{equation*}
	since $\Phi_{\mK}$ is nonsingular, where $\mK = \{(s_{1}, a_{1}), \cdots, (s_{K}, a_{K})\}$. Next for $k\in [K]$, we iteratively choose $\phi'_{K+1}(s_{k}, a_{k})\ge N(s_{k}, a_{k})$ and add it to the $k$-th row of the feature matrix $\Phi_{\mK'}'$ such that the matrix is still nonsingular. (According to Lemma \ref{lem4}, such $\phi'_{K+1}(s_{k}, a_{k})$ exists.) After these $K$ operations, we have
	\begin{equation*}
		\det \Phi'_{\mK'} = \det\begin{bmatrix}\phi_{1}'(s_{1}, a_{1}) & \cdots & \phi_{K}'(s_{1}, a_{1}) & \phi_{K+1}'(s_{1}, a_{1})\\\phi_{1}'(s_{2}, a_{2}) & \cdots & \phi_{K}'(s_{2}, a_{2}) & \phi_{K+1}'(s_{2}, a_{2})\\\vdots & \ddots & \vdots & \vdots\\\phi_{1}'(s_{K}, a_{K}) & \cdots & \phi_{K}'(s_{K}, a_{K}) & \phi_{K+1}'(s_{K}, a_{K})\\ \phi_{1}'(s', a') & \cdots & \phi_{K}'(s', a') & \phi_{K+1}'(s', a')\end{bmatrix}\neq 0,
	\end{equation*}
	where $\phi_{k}'(s_{k}, a_{k}) = \phi_{k}(s_{k}, a_{k}) + \phi_{K+1}'(s_{k}, a_{k})$. This indicates that $\Phi'_{\mK'}$ is nonsingular. Next for $(s, a)$ not in $\mK'$, we choose $\phi_{K+1}'(s, a) = N(s, a)$ and let $\phi'_{k}(s, a) = \phi_{k}(s, a) + \phi_{K+1}'(s, a)$. 
	\par Then we have $\phi'(s, a)\ge 0$ for any $(s, a)\in\mS\times\mA$ and $\Phi_{\mK'}'$ is nonsingular. Finally we normalize all features such that $\|\phi'(s, a)\|_{1}\le 1$ while keeping $\phi'(s, a)\ge 0$ for any $(s, a)\in\mS\times\mA$. And Assumption \ref{ass1} holds for $L = \|(\Phi_{\mK}')^{-1}\|_{\infty}$.  
\end{proof}

\section{Proof of Theorem \ref{thm2}}
\par In this section, we present the formal proof of Theorem \ref{thm2}. In the following, we assume that all features $\phi(s, a)$ are nonnegative, and Assumption \ref{ass1} holds for all the time.

\subsection{Notations}
\par We define the following $\mT$-operators and $Q$ functions:
\begin{equation*}
	\begin{aligned}
		& [\mT V](s) = \begin{cases} \max_{a\in\mA} r(s, a) + \gamma P(\cdot|s, a)^{T}V, \quad\forall s\in\mS_{1},\\\min_{a\in\mA} r(s, a) + \gamma P(\cdot|s, a)^{T}V, \quad\forall s\in\mS_{2},\end{cases}\\
		& [\mT_{\pi} V](s) = r(s, \pi(s)) + \gamma P(\cdot|s, \pi(s))^{T}V,
	\end{aligned}
\end{equation*}
\begin{equation*}
	\begin{aligned}
		& Q_{\theta^{(i, j)}}(s, a) = r(s, a) + \gamma\phi(s, a)^{T}\overline{w}^{(i, j)},\\
		& \overline{Q}_{\theta^{(i, j)}}(s, a) = r(s, a) + \gamma P(\cdot|s, a)^{T}V_{\theta^{(i, j-1)}}.
	\end{aligned}
\end{equation*}
\par For these $\mT$ and $\mT_{\pi}$ operators, we have the following monotonicity and contraction property:
\begin{proposition}
	For any value function $V, V'$, if $V\le V'$, we have
	\begin{equation*}
		\begin{aligned}
			\mT V\le \mT V', &\quad \mT_{\pi} V\le \mT_{\pi} V'\\
			\|\mT V - \mT V'\|_{\infty}\le \gamma\|V - V'\|_{\infty}, &\quad\|\mT_{\pi} V - \mT_{\pi} V'\|_{\infty}\le \gamma\|V - V'\|_{\infty}
		\end{aligned}
	\end{equation*}
	for any strategy $\pi$. Furthermore, $v^{*}$ and $V^{\pi}$ are unique fixed points of $\mT$ and $\mT_{\pi}$, respectively, and
	\begin{equation*}
		\lim_{t\to\infty}\mT^{t}V = v^{*}, \quad \lim_{t\to\infty}\mT^{t}_{\pi}V = v^{\pi}.
	\end{equation*}
\end{proposition}
\par Next we use the following events to describe properties of our algorithm.
\begin{itemize}
	\item Let $\mG^{(i)}$ to be the event
	\begin{equation*}
		\begin{aligned}
			& 0\le V_{\theta^{(i, 0)}}(s)\le \left[\mathcal{T}V_{\theta^{(i, 0)}}\right](s)\le v^{*}(s),\\
			& V_{\theta^{(i, 0)}}(s)\le \left[\mathcal{T}_{\pi_{\theta^{(i, 0)}}}V_{\theta^{(i, 0)}}\right](s),\\
			& v^{*}(s) - V_{\theta^{(i, 0)}}(s)\le \frac{2^{-i}}{1 - \gamma};
		\end{aligned}
	\end{equation*}
	\item $\mE^{(i, 0)}$ to be the event of
	\begin{equation*}
		\|w^{(i, 0)} - \Phi_{\mK}^{-1}P_{\mK}V_{\theta^{(i, 0)}}\|_{\infty}\le \epsilon_{1},
	\end{equation*}
	where $w^{(i, 0)} = \Phi_{\mK}^{-1}M^{(i, 0)}$;
	\item $\mE^{(i, j)}$ to be the event of
	\begin{equation*}
		\|w^{(i, j)} - w^{(i, 0)} - \Phi_{\mK}^{-1}P_{\mK}(V_{\theta^{(i, j-1)}} - V_{\theta^{(i, 0)}})\|_{\infty}\le \Theta\left[\frac{L\cdot 2^{-i}}{1-\gamma}\sqrt{\frac{\log(RR'K\delta^{-1})}{m_{1}}}\right].
	\end{equation*}
\end{itemize}

\subsection{Preserving the Monotonicity}
\par We first present several lemmas to establish some properties of $\mT$ and $\mT_{\pi_{\theta^{(i, j)}}}$ on $V_{\theta^{(i, j)}}$.
\begin{lemma}\label{lem1}
	Suppose $\mG^{(i)}, \mE^{(i, 0)}, \cdots, \mE^{(i, j)}$ holds. We have
	\begin{equation*}
		\begin{aligned}
			& 0\le V_{\theta^{(i, j')}}(s)\le \left[\mathcal{T}V_{\theta^{(i, j')}}\right](s)\le v^{*}(s),\\
			& V_{\theta^{(i, j')}}(s)\le \left[\mathcal{T}_{\pi_{\theta^{(i, j')}}}V_{\theta^{(i, j')}}\right](s),\\
		\end{aligned}
	\end{equation*}
	for $\forall 0\le j'\le j$.
\end{lemma}
\begin{proof}
	We prove this result by induction. When $j' = 0$, these conditions already hold according to the event $\mG^{(i)}$. Now assuming these conditions hold for $j'-1\ge 0$, we consider the case of $j'$.
	\par According to the construction of $V_{\theta^{(i, j')}}$,
	\begin{equation}\label{mon1}
		\begin{aligned}
			& V_{\theta^{(i, j')}}(s) = \max\left\{V_{\theta^{(i, j'-1)}}(s), \max_{a\in\mA_{s}}Q_{\theta^{(i, j)}}(s, a)\right\}\ge V_{\theta^{(i, j'-1)}}(s), \quad \forall s\in\mS_{1},\\
			& V_{\theta^{(i, j')}}(s) = \max\left\{V_{\theta^{(i, j'-1)}}(s), \min_{a\in\mA_{s}}Q_{\theta^{(i, j)}}(s, a)\right\}\ge V_{\theta^{(i, j'-1)}}(s), \quad \forall s\in\mS_{2},\\
		\end{aligned}
	\end{equation}
	Hence for any $s\in\mS$, there are two cases:
	\begin{enumerate}
		\item $V_{\theta^{(i, j')}}(s) = V_{\theta^{(i, j'-1)}}(s)$. Then $\pi_{\theta^{(i, j')}}(s) = \pi_{\theta^{(i, j'-1)}}(s)$;
		\item \begin{enumerate}
			\item $V_{\theta^{(i, j')}}(s) = \max_{a}Q_{\theta^{(i, j')}}(s, a)$ and $\pi_{\theta^{(i, j')}} = \arg\max_{a\in\mA_{s}}Q_{\theta^{(i, j')}}(s, a)$ if $s\in\mS_{1}$;
			\item $V_{\theta^{(i, j')}}(s) = \min_{a}Q_{\theta^{(i, j')}}(s, a)$ and $\pi_{\theta^{(i, j')}} = \arg\min_{a\in\mA_{s}}Q_{\theta^{(i, j')}}(s, a)$ if $s\in\mS_{2}$.
		\end{enumerate}
	\end{enumerate}
	\par In the first case, according to induction results and the monotonicity of $\mT, \mT_{\pi}$, we have
	\begin{equation*}
		\begin{aligned}
			& V_{\theta^{(i, j')}}(s) = V_{\theta^{(i, j'-1)}}(s) \le [\mT V_{\theta^{(i, j'-1)}}](s) \le [\mT V_{\theta^{(i, j')}}](s)\\
			& V_{\theta^{(i, j')}}(s) = V_{\theta^{(i, j'-1)}}(s) \le \left[\mT_{\pi_{\theta^{(i, j'-1)}}} V_{\theta^{(i, j'-1)}}\right](s) = \left[\mT_{\pi_{\theta^{(i, j')}}} V_{\theta^{(i, j'-1)}}\right](s) \le \left[\mT_{\pi_{\theta^{(i, j')}}} V_{\theta^{(i, j')}}\right](s)
		\end{aligned}
	\end{equation*}
	In the second case, according to the event $\mE^{(i, 0)}, \mE^{(i, j')}$, we obtain
	\begin{equation*}
		\begin{aligned}
			& \|w^{(i, j')} - \Phi_{\mK}^{-1}P_{\mK}V_{\theta^{(i, j'-1)}}\|_{\infty}\\
			\le &\ \|w^{(i, j')} - w^{(i, 0)} - \Phi_{\mK}^{-1}P_{\mK}(V_{\theta^{(i, j'-1)}} - V_{\theta^{(i, 0)}})\|_{\infty} + \|w^{(i, 0)} - \Phi_{\mK}^{-1}P_{\mK}V_{\theta^{(i, 0)}}\|_{\infty}\\
			\le &\ \Theta\left[\frac{L\cdot 2^{-i}}{1-\gamma}\sqrt{\frac{\log(RR'K\delta^{-1})}{m_{1}}}\right] + \epsilon_{1}\\
			= &\ \epsilon^{(i)},
		\end{aligned}
	\end{equation*}
	which indicates that
	\begin{equation*}
		\overline{w}^{(i, j')} - \Phi_{\mK}^{-1}P_{\mK}V_{\theta^{(i, j'-1)}} = w^{(i, j')} - \epsilon^{(i)} - \Phi_{\mK}^{-1}P_{\mK}V_{\theta^{(i, j'-1)}}\le 0.
	\end{equation*}
	Since all features are nonnegative, we have the following inequality for $\forall a\in\mA_{s}$.
	\begin{equation*}
		\begin{aligned}
			Q_{\theta^{(i, j')}}(s, a) & = r(s, a) + \gamma \phi(s, a)^{T}\overline{w}^{(i, j')} \le r(s, a) + \gamma\phi(s, a)^{T}\Phi_{\mK}^{-1}P_{\mK}V_{\theta^{(i, j'-1)}}\\
			& = r(s, a) + \gamma P(\cdot|s, a)^{T}V_{\theta^{(i, j'-1)}} = \overline{Q}_{\theta^{(i, j')}}(s, a).
		\end{aligned}
	\end{equation*}
	Therefore, we have
	\begin{equation*}
		\begin{aligned}
			& V_{\theta^{(i, j')}}(s) = \max_{a\in\mathcal{A}_{s}}Q_{\theta^{(i, j')}}(s, a)\le \max_{a\in\mathcal{A}_{s}}\overline{Q}_{\theta^{(i, j')}}(s, a) = \left[\mathcal{T}V_{\theta^{(i, j'-1)}}\right](s),\quad \text{if } s\in\mathcal{S}_{1};\\
			& V_{\theta^{(i, j')}}(s) = \min_{a\in\mathcal{A}_{s}}Q_{\theta^{(i, j')}}(s, a)\le \min_{a\in\mathcal{A}_{s}}\overline{Q}_{\theta^{(i, j')}}(s, a) = \left[\mathcal{T}V_{\theta^{(i, j'-1)}}\right](s),\quad \text{if } s\in\mathcal{S}_{2};\\
			& V_{\theta^{(i, j')}}(s) = Q_{\theta^{(i, j')}}(s, \pi_{\theta^{(i, j')}}(s))\le \overline{Q}_{\theta^{(i, j')}}(s, \pi_{\theta^{(i, j')}}(s)) = \left[\mathcal{T}_{\pi_{\theta^{(i, j')}}}V_{\theta^{(i, j'-1)}}\right](s).
		\end{aligned}
	\end{equation*}
	Noticing $V_{\theta^{(i, j'-1)}}\le V_{\theta^{(i, j')}}$ and the monotonicity of $\mathcal{T}$ and $\mathcal{T}_{\pi_{\theta^{(i, j'-1)}}}$, we have
	\begin{equation*}
		\begin{aligned}
			& V_{\theta^{(i, j')}}(s)\le \left[\mathcal{T}V_{\theta^{(i, j'-1)}}\right](s)\le \left[\mathcal{T}V_{\theta^{(i, j')}}\right](s),\\
			& V_{\theta^{(i, j')}}(s)\le \left[\mathcal{T}_{\pi_{\theta^{(i, j'-1)}}}V_{\theta^{(i, j')}}\right](s).
		\end{aligned}
	\end{equation*}
	\par Therefore, for all $s\in\mathcal{S}$, 
	\begin{equation*}
		\begin{aligned}
			& V_{\theta^{(i, j')}}(s)\le \left[\mT V_{\theta^{(i, j')}}\right](s),\\
			& V_{\theta^{(i, j')}}(s)\le \left[\mT_{\pi_{\theta^{(i, j'-1)}}}V_{\theta^{(i, j')}}\right](s).
		\end{aligned}
	\end{equation*}
	Again according to the monotonicity of $\mT$, we have
	\begin{equation*}
		V_{\theta^{(i, j')}}(s)\le \left[\mT V_{\theta^{(i, j')}}\right](s)\le \left[\mT^{2}V_{\theta^{(i, j')}}\right](s)\le \cdots\le v^{*}(s).
	\end{equation*}
	This completes the induction. 
\end{proof}

\par Next, we will exhibit an approximate contraction property of our algorithm.
\begin{lemma}\label{lem2}
	If $\mG^{(i)}, \mE^{(i, 0)}, \cdots, \mE^{(i, j)}$ holds, then for $1\le j'\le j$ we have
	\begin{equation*}
		v^{*}(s) - V_{\theta^{(i, j')}}(s)\le \max_{a\in\mathcal{A}_{s}}\left[\gamma P(\cdot|s, a)^{T}(v^{*} - V_{\theta^{(i, j'-1)}})\right] + 2\epsilon^{(i)}, \quad\forall s\in\mS.
	\end{equation*}
\end{lemma}
\begin{proof}
	According to Algorithm \ref{alg2},
	\begin{equation*}
		\overline{w}^{(i, j')} = w^{(i, j')} - \epsilon^{(i)}\ge \Phi_{\mK}^{-1}P_{\mK}^{T}V_{\theta^{(i, j'-1)}} - 2\epsilon^{(i)}.
	\end{equation*}
	Using $V_{\theta^{(i, j'-1)}}\le V_{\theta^{(i, j')}}\le v^{*}$ in Lemma \ref{lem1}, we have for $\forall s\in\mathcal{S}_{1}$,
	\begin{equation*}
		\begin{aligned}
			v^{*}(s) - V_{\theta^{(i, j')}}(s) & \le v^{*}(s) - \max_{a\in\mA_{s}}\left[r(s, a) + \gamma\phi(s, a)^{T}\overline{w}^{(i, j')}\right]\\
			& \le v^{*}(s) - \max_{a\in\mathcal{A}_{s}}\left[r(s, a) + \gamma\phi(s, a)^{T}(\Phi_{\mK}^{-1}P_{\mK}^{T}V_{\theta^{(i, j'-1)}} - 2\epsilon^{(i)})\right]\\
			& = \max_{a\in\mA_{s}}\left[r(s, a) + \gamma P(\cdot|s, a)^{T}v^{*}\right] - \max_{a\in\mathcal{A}_{s}}\left[r(s, a) + \gamma\phi(s, a)^{T}(\Phi_{\mK}^{-1}P_{\mK}^{T}V_{\theta^{(i, j'-1)}} - 2\epsilon^{(i)})\right]\\
			& \le \max_{a\in\mathcal{A}_{s}}\left[\gamma P(\cdot|s, a)^{T}v^{*} - \gamma P(\cdot|s, a)^{T}V_{\theta^{(i, j'-1)}} + 2\epsilon^{(i)}\|\phi(s, a)\|_{1}\right]\\
			& \le \max_{a\in\mathcal{A}_{s}}\left[\gamma P(\cdot|s, a)^{T}(v^{*} - V_{\theta^{(i, j'-1)}})\right] + 2\epsilon^{(i)},
		\end{aligned}
	\end{equation*}
	where the first inequality is due to $V_{\theta^{(i, j')}}(s)\ge \max_{a\in\mA_{s}}\left[r(s, a) + \gamma\phi(s, a)^{T}\overline{w}^{(i, j')}\right]$, and the third equality is due to the property of $v^{*}$. And for $s\in\mS_{2}$, similarly we have
	\begin{equation*}
		\begin{aligned}
			v^{*}(s) - V_{\theta^{(i, j')}}(s) & \le v^{*}(s) - \min_{a\in\mA_{s}}[r(s, a) + \gamma\phi(s, a)^{T}\overline{w}^{(i, j')}]\\
			& \le v^{*}(s) - \min_{a\in\mathcal{A}_{s}}\left[r(s, a) + \gamma\phi(s, a)^{T}(\Phi_{\mK}^{-1}P_{\mK}^{T}V_{\theta^{(i, j'-1)}} - 2\epsilon^{(i)})\right]\\
			& = \min_{a\in\mA_{s}}\left[r(s, a) + P(\cdot|s, a)^{T}v^{*}\right] - \min_{a\in\mathcal{A}_{s}}\left[r(s, a) + \gamma\phi(s, a)^{T}(\Phi_{\mK}^{-1}P_{\mK}^{T}V_{\theta^{(i, j'-1)}} - 2\epsilon^{(i)})\right]\\
			& \le \max_{a\in\mathcal{A}_{s}}\left[\gamma P(\cdot|s, a)^{T}v^{*} - \gamma P(\cdot|s, a)^{T}V_{\theta^{(i, j'-1)}} + 2\epsilon^{(i)}\|\phi(s, a)\|_{1}\right]\\
			& \le \max_{a\in\mathcal{A}_{s}}\left[\gamma P(\cdot|s, a)^{T}(v^{*} - V_{\theta^{(i, j'-1)}})\right] + 2\epsilon^{(i)}.
		\end{aligned}
	\end{equation*}
	The proof is completed.
\end{proof}

\subsection{Analysis of the Confidence Bounds}
\par In this subsection, we analyze the confidence bound of sampling.
\begin{lemma}\label{lem3}
	For $0\le i\le R'$, 
	\begin{equation*}
		\pr\left(\mE^{(i, 0)}, \cdots, \mE^{(i, R)} | \mG^{(i)}\right)\le 1 - \frac{\delta}{R'}
	\end{equation*}
\end{lemma}
\begin{proof}
	\par Conditioned on $\mG^{(i)}$, we have 
	\begin{equation*}
		\begin{aligned}
			& 0\le V_{\theta^{(i, 0)}}(s)\le v^{*}(s)\le \frac{1}{1-\gamma},\\
			& v^{*}(s) - V_{\theta^{(i, 0)}}(s)\le \frac{2^{-i}}{1-\gamma}.
		\end{aligned}
	\end{equation*}
	For any $k\in [K]$ and $\delta\in(0, 1)$, according to Hoeffding inequality, with probability at least $1 - \delta$,
	\begin{equation*}
		|M^{(i, 0)}(k) - P(\cdot|s_{k}, a_{k})V_{\theta^{(i, 0)}}|\le c_{1}\cdot\max_{s\in\mS}|V_{\theta^{(i, 0)}}|\cdot\sqrt{\frac{\log[\delta^{-1}]}{m}}\le c_{1}\cdot\frac{1}{1-\gamma}\sqrt{\frac{\log[\delta^{-1}]}{m}}
	\end{equation*}
	holds for some constant $c_{1}$. If we switch $\delta$ into $\delta/(RR'K)$, then we obtain
	\begin{equation*}
		|M^{(i, 0)}(k) - P(\cdot|s_{k}, a_{k})V_{\theta^{(i, 0)}}|\le \epsilon_{1} / L
	\end{equation*}
	holds with probability at least $1-\delta/(RR'K)$. Next using the fact $\|\Phi_{\mK}\|_{\infty}\le L$ and apply the union bound for all $k\in [K]$, we have
	\begin{equation*}
		\|w^{(i, 0)} - \Phi_{\mK}^{-1}P_{\mK}V_{\theta^{(i, 0)}}\|_{\infty} = \|\Phi_{\mK}^{-1}(M^{(i, 0)} - P_{\mK}V_{\theta^{(i, 0)}})\|_{\infty}\le \epsilon_{1}
	\end{equation*}
	holds with probability at least $1 - \delta / (RR')$. This indicates that $\mE^{(i, 0)}$ holds with probability at least $1 - \delta / (RR')$. 
	\par As for $\mE^{(i, 1)}$, since $M^{(i, 1)} = M^{(i, 0)}$ and $w^{(i, 1)} = w^{(i, 0)}$, the event $\mE^{(i, 1)}$ holds with probability 1.
	\par For $2\le j\le R$, again using the Hoeffding inequality and the event $\mG^{(i)}$, we have
	\begin{equation*}
		\begin{aligned}
			&\ |M^{(i, j)}(k) - M^{(i, 0)}(k) - P(\cdot|s_{k}, a_{k})^{T}\left(V_{\theta^{(i, j-1)}} - V_{\theta^{(i, 0)}}\right)|\\
			= &\ \left|\frac{1}{m_{1}}\sum_{l=1}^{m_{1}}\left(V_{\theta^{(i, j-1)}}(x_{k}^{(l)}) - V_{\theta^{(i, 0)}}(x_{k}^{(l)})\right) - P(\cdot|s_{k}, a_{k})^{T}\left(V_{\theta^{(i, j-1)}} - V_{\theta^{(i, 0)}}\right)\right|\\
			\le &\ c_{1}\max_{s\in\mS}|V_{\theta^{(i, j-1)}}(s) - V_{\theta^{(i, 0)}}(s)|\cdot\sqrt{\frac{\log(\delta^{-1})}{m_{1}}}\\
			\le &\ c_{1}\max_{s\in\mS}|v^{*}(s) - V_{\theta^{(i, 0)}}(s)|\cdot\sqrt{\frac{\log(\delta^{-1})}{m_{1}}}\\
			\le &\ c_{1}\cdot\frac{2^{-i}}{1-\gamma}\sqrt{\frac{\log(\delta^{-1})}{m_{1}}}
		\end{aligned}
	\end{equation*}
	with probability at least $1 - \delta$. Since $w^{(i, j)} = \Phi_{\mK}^{-1}M^{(i, j)}, w^{(i, 0)} = \Phi_{\mK}^{-1}M^{(i, 0)}$, we switch $\delta$ into $\delta / (RR'K)$ and apply the union bound for all $k\in [K]$ to obtain that
	\begin{equation*}
		\begin{aligned}
			& \|w^{(i, j)} - w^{(i, 0)} - \Phi_{\mK}^{-1}P_{\mK}(V_{\theta^{(i, j-1)}} - V_{\theta^{(i, 0)}})\|_{\infty}\\
			= &\ \|\Phi_{\mK}^{-1}\left(M^{(i, j)} - M^{(i, 0)} - P_{\mK}(V_{\theta^{(i, j-1)}} - V_{\theta^{(i, 0)}})\right)\|_{\infty}\\
			\le &\ L\cdot \|M^{(i, j)} - M^{(i, 0)} - P_{\mK}(V_{\theta^{(i, j-1)}} - V_{\theta^{(i, 0)}})\|_{\infty}\\
			\le &\ \Theta\left[\frac{L\cdot 2^{-i}}{1-\gamma}\sqrt{\frac{\log(RR'K\delta^{-1})}{m_{1}}}\right]
		\end{aligned}
	\end{equation*}
	holds with probability at least $1 - \delta / (RR')$. So is the probability of $\mE^{(i, j)}$ conditioned on $\mG^{(i)}$. Applying the union bound for all $\mE^{(i, 0)}, \mE^{(i, 1)}, \mE^{(i, 2)}\cdots, \mE^{(i, R)}$, we have
	\begin{equation*}
		\pr\left(\mE^{(i, 0)}, \cdots, \mE^{(i, R)} | \mG^{(i)}\right)\le 1 - \frac{\delta}{R'},
	\end{equation*}
	which completes the proof.
\end{proof}

\subsection{Analysis of the Error in the Next Iteration}
\begin{lemma}\label{next}
	For $\forall 1\le i\le R'$, we have
	\begin{equation*}
		\pr(\mG^{(i+1)}, \mE^{(i, 0)}, \cdots, \mE^{(i, R)}|\mG^{(i)})\ge 1-\frac{\delta}{R'}.
	\end{equation*}
\end{lemma}
\begin{proof}
	Conditioned on $\mG^{(i)}$, suppose $\mE^{(i, 0)}, \cdots, \mE^{(i, R)}$ all hold. Using Lemma \ref{lem2} for $R$ times, there exists a constant $C$ such that
	\begin{equation*}
		\begin{aligned}
			\|v^{*} - V_{\theta^{(i, R)}}\|_{\infty}& \le \max_{s\in\mS, a\in\mA}\left[\gamma P(\cdot|s, a)^{T}(v^{*} - V_{\theta^{(i, R-1)}})\right] + 2\epsilon^{(i)}\\
			&\le \gamma\|v^{*} - V_{\theta^{(i, R-1)}}\|_{\infty} + 2\epsilon^{(i)}\\
			&\le \cdots\\
			&\le \gamma^{R}\|v^{*} - V_{\theta^{(i, 0)}}\|_{\infty} + 2\sum_{j'=0}^{R-1}\gamma^{j'}\epsilon^{(i)}\\
			&\le \frac{\gamma^{R}}{1-\gamma} + \frac{C}{1-\gamma}\left[\frac{2^{-i}L}{1-\gamma}\sqrt{\frac{\log(RR'K\delta^{-1})}{m_{1}}} + \frac{L}{1-\gamma}\sqrt{\frac{\log(RR'K\delta^{-1})}{m}}\right],
		\end{aligned}
	\end{equation*}
	where the last inequality is due to events $\mE^{(i, 0)}, \cdots, \mE^{(i, R)}$ and $\|v^{*} - V_{\theta^{(i, 0)}}\|_{\infty}\le 1/(1-\gamma)$ according to $\mG^{(i)}$. Hence if we choose
	\begin{equation*}
		\begin{aligned}
			& R = C_{R}\cdot\frac{\log(1/(\epsilon(1-\gamma)))}{1-\gamma}, \quad m = C_{1}\cdot\frac{L^{2}\log(R'RK\delta^{-1})}{\epsilon^{2}(1-\gamma)^{4}},\\
			& m_{1} = C_{2}\cdot\frac{L^{2}\log(R'RK\delta^{-1})}{(1-\gamma)^{2}}, \quad\epsilon\le\frac{2^{-i}}{1-\gamma},
		\end{aligned}
	\end{equation*}
	where $C_{R}, C_{1}, C_{2}$ are constant numbers, we will have
	\begin{equation*}
		\|v^{*} - V_{\theta^{(i, R)}}\|_{\infty}\le \frac{2^{-i-1}}{1-\gamma}.
	\end{equation*}
	Furthermore, since $\theta^{(i+1, 0)} = \theta^{(i, R)}$, the following inequalities
	\begin{equation*}
		\begin{aligned}
			& 0\le V_{\theta^{(i+1, 0)}}(s)\le \left[\mathcal{T}V_{\theta^{(i+1, 0)}}\right](s)\le v^{*}(s),\\
			& V_{\theta^{(i+1, 0)}}(s)\le \left[\mathcal{T}_{\pi_{\theta^{(i+1, 0)}}}V_{\theta^{(i+1, 0)}}\right](s)\\
			& v^{*}(s) - V_{\theta^{(i+1, 0)}}(s)\le \frac{2^{-i-1}}{1 - \gamma}
		\end{aligned}
	\end{equation*}
	hold according to Lemma \ref{lem1}. Hence the event $\mG^{(i+1)}$ holds when $\mG^{(i)}, \mE^{(i, 0)}, \cdots, \mE^{(i, R)}$ all hold. Therefore, according to Lemma \ref{lem3} we have proved that
	\begin{equation*}
		\pr(\mG^{(i+1)}, \mE^{(i, 0)}, \cdots, \mE^{(i, R)}|\mG^{(i)})\ge 1-\frac{\delta}{R'}.
	\end{equation*}
\end{proof}

\subsection{Analysis of Approximation from Two Sides}
\begin{lemma}\label{lemma1}
	With probability at least $1 - \delta$, $\mG^{(0)}, \mG^{(i)}, \mE^{(i-1, j)}$ hold for $\forall 1\le i\le R', 0\le j\le R$ and $V_{\theta^{(R', R)}}$ is an $\epsilon$-optimal value.
\end{lemma}
\begin{proof}
	First of all, according to the initialization,
	\begin{equation*}
		V_{\theta^{(0, 0)}}(s) = \begin{cases}
			\max_{a\in\mA_{s}}r(s, a)\ge 0, \quad\text{if }s\in\mS_{1},\\
			\min_{a\in\mA_{s}}r(s, a)\ge 0, \quad\text{if }s\in\mS_{2}.\\
		\end{cases}
	\end{equation*}
	\par This indicates that
	\begin{equation*}
		\left[\mT V_{\theta^{(0, 0)}}\right](s) = \begin{cases}
			\max_{a\in\mA_{s}}r(s, a) + \gamma P(\cdot|s, a)^{T}V_{\theta^{(0, 0)}}\ge \max_{a\in\mA_{s}}r(s, a) = V_{\theta^{(0, 0)}}(s), \quad\text{if }s\in\mS_{1},\\
			\min_{a\in\mA_{s}}r(s, a) + \gamma P(\cdot|s, a)^{T}V_{\theta^{(0, 0)}}\ge \min_{a\in\mA_{s}}r(s, a) = V_{\theta^{(0, 0)}}(s), \quad\text{if }s\in\mS_{2}.\\
		\end{cases}
	\end{equation*}
	and
	\begin{equation*}
		\left[\mT_{\pi_{\theta^{(0, 0)}}}V_{\theta^{(0, 0)}}\right](s) = r(s, \pi_{\theta^{(0, 0)}}) + \gamma P(\cdot|s, \pi_{\theta^{(0, 0)}})V_{\theta^{(0, 0)}}(s)\ge r(s, \pi_{\theta^{(0, 0)}}) = V_{\theta^{(0, 0)}}(s).
	\end{equation*}
	According to the monotonicity of $\mT$,
	\begin{equation*}
		0\le V_{\theta^{(0, 0)}}\le \mT V_{\theta^{(0, 0)}}\le \cdots\le v^{*}\le \frac{1}{1-\gamma}.
	\end{equation*}
	Hence $\mG^{(0)}$ always holds.
	\par Next based on Lemma \ref{next}, we have
	\begin{equation*}
		\begin{aligned}
			&\ \pr(\mG^{(i)}, \mE^{(i-1, j)}, \forall 1\le i\le R', 0\le j\le R|\mG^{(0)})\\
			= &\ \pr(\mG^{(i)}, \mE^{(i-1, j)}, \forall 1\le i\le R'-1, 0\le j\le R|\mG^{(0)})\pr(\mG^{(R')}, \mE^{(R'-1, 0)}, \cdots, \mE^{(R'-1, R)}|\mG^{(R'-1)})\\
			\ge &\ \pr(\mG^{(i)}, \mE^{(i-1, j)}, \forall 1\le i\le R'-1, 0\le j\le R|\mG^{(0)})\cdot (1 - \delta/R')\\
			\ge &\ \pr(\mG^{(i)}, \mE^{(i-1, j)}, \forall 1\le i\le R'-2, 0\le j\le R|\mG^{(0)})\cdot (1 - \delta/R')^{2}\\
			\ge &\ \cdots\\
			\ge &\ (1 - \delta / R')^{R'}\\
			\ge &\ 1 - \delta.
		\end{aligned}
	\end{equation*}
	If $\mG^{(R')}$ holds, then we obtain
	\begin{equation*}
		\|v^{*} - V_{\theta^{(R, R')}}\|_{\infty}\le \frac{2^{-R'}}{1-\gamma}.
	\end{equation*}
	Hence when choosing $R' = \Theta[\log(\epsilon^{-1}(1-\gamma)^{-1})]$, we have
	\begin{equation*}
		\|v^{*} - V_{\theta^{(R, R')}}\|_{\infty}\le\epsilon,
	\end{equation*}
	which indicates that $V_{\theta^{(R, R')}}$ is an $\epsilon$-optimal value. Therefore, the event that $V_{\theta^{(R, R')}}$ is an $\epsilon$-optimal value, together with $\mG^{(i)}, \mE^{(i, j)}$, happens with probability at least $1 - \delta$.
\end{proof}

\par Next, we provide some notations and lemmas for the 2-TBSG $\mM'$. Suppose the equilibrium value function of $\mM'$ is $v'$. For $\eta = \{z^{(h)}\}_{h = 1}^{Z}$, we define
\begin{equation*}
	\begin{aligned}
		& V_{\eta}'(s) = \begin{cases}\max_{h\in[Z]}\min_{a\in\mathbb{A}_{s}}\left(r'(s, a) + \gamma \phi(s, a)^{T}z^{(h)}\right), \quad \forall s\in\mathcal{S}_{1},\\\max_{h\in[Z]}\max_{a\in\mathbb{A}_{s}}\left(r'(s, a) + \gamma \phi(s, a)^{T}z^{(h)}\right), \quad \forall s\in\mathcal{S}_{2},\end{cases}\\
		& \pi_{\eta}'(s) = \begin{cases}\max_{h\in[Z]}\left[\arg\min_{a\in\mA_{s}}\left(r'(s, a) + \gamma \phi(s, a)^{T}z^{(h)}\right)\right], \quad \forall s\in\mathcal{S}_{1},\\\max_{h\in[Z]}\left[\arg\max_{a\in\mA_{s}}\left(r'(s, a) + \gamma \phi(s, a)^{T}z^{(h)}\right)\right], \quad \forall s\in\mathcal{S}_{2},\end{cases}\\
		& W_{\eta}(s) = \frac{1}{1-\gamma} - V_{\eta}'(s).
	\end{aligned}
\end{equation*}
To describe the connection between values of $\mM$ and $\mM'$, we introduce the following lemma, whose proof is obvious.
\begin{lemma}
	\par $v'(s) = 1/(1-\gamma) - v(s)$. Any equilibrium strategy for $\mM'$ is also an equilibrium strategy for $\mM$. And if $\|V'_{\eta} - v'\|\le \epsilon$, then $W_{\eta}$ is an $\epsilon$-optimal value.
\end{lemma}
Next we define events $\mG^{(0)}, \mG^{(i)}_{1}, \mE^{(i-1, j)}_{1}$ for $1\le i\le R', 0\le j\le R$ similar to the case of $\mM$.
\begin{itemize}
	\item Let $\mG^{(i)}_{1}$ to be the event
	\begin{equation*}
		\begin{aligned}
			& 0\le V_{\eta^{(i, 0)}}'(s)\le \left[\mathcal{T}'V_{\eta^{(i, 0)}}'\right](s)\le v'(s),\\
			& V_{\eta^{(i, 0)}}'(s)\le \left[\mathcal{T}'_{\pi'_{\eta^{(i, 0)}}}V_{\eta^{(i, 0)}}'\right](s),\\
			& v'(s) - V_{\eta^{(i, 0)}}'(s)\le \frac{2^{-i}}{1 - \gamma},
		\end{aligned}
	\end{equation*}
	where $\mT'$ and $\mT'_{\pi'}$ is defined as
	\begin{equation}
		\begin{aligned}
			& [\mT'V'](s) = \begin{cases} \min_{a\in\mA_{s}}[r'(s, a) + \gamma P(\cdot|s, a)^{T}V], &\quad \forall s\in\mS_{1},\\\max_{a\in\mA_{s}}[r'(s, a) + \gamma P(\cdot|s, a)^{T}V], &\quad \forall s\in\mS_{2}.\end{cases}\\
			& [\mT'_{\pi'}V'](s) = r'(s, \pi'(s)) + \gamma P(\cdot|s, \pi'(s))^{T}V', \quad \forall s\in\mS.
		\end{aligned}
	\end{equation}
	This event is equivalent to the event
	\begin{equation*}
		\begin{aligned}
			& \frac{1}{1-\gamma}\ge W_{\eta^{(i, 0)}}(s)\ge \left[\mathcal{T}W_{\eta^{(i, 0)}}\right](s)\ge v^{*}(s),\\
			& W_{\eta^{(i, 0)}}(s)\ge \left[\mathcal{T}_{\pi'_{\eta^{(i, 0)}}}W_{\eta^{(i, 0)}}\right](s),\\
			& v^{*}(s) - W_{\eta^{(i, 0)}}(s)\ge \frac{2^{-i}}{1 - \gamma};
		\end{aligned}
	\end{equation*}
	\item $\mE^{(i, 0)}_{1}$ to be the event of
	\begin{equation*}
		\|z^{(i, 0)} - \Phi_{\mK}^{-1}P_{\mK}V_{\eta^{(i, 0)}}'\|_{\infty}\le \epsilon_{1};
	\end{equation*}
	\item $\mE^{(i, j)}_{1}$ to be the event of
	\begin{equation*}
		\|z^{(i, j)} - z^{(i, 0)} - \Phi_{\mK}^{-1}P_{\mK}(V_{\eta^{(i, j-1)}}' - V_{\eta^{(i, 0)}}')\|_{\infty}\le \Theta\left[\frac{L\cdot 2^{-i}}{1-\gamma}\sqrt{\frac{\log(RR'K\delta^{-1})}{m_{1}}}\right].
	\end{equation*}
\end{itemize}
\par We present two following lemmas, which can be viewed as counterparts for $W$ of Lemma \ref{lem1} and Lemma \ref{lemma1}.
\begin{lemma}\label{lem5}
	Suppose $\mG^{(i)}_{1}, \mE^{(i, 0)}_{1}, \cdots, \mE^{(i, j)}_{1}$ holds. We have
	\begin{equation*}
		\begin{aligned}
			& 1/(1-\gamma)\ge W_{\theta^{(i, j')}}(s)\ge \left[\mathcal{T}W_{\theta^{(i, j')}}\right](s)\ge v^{*}(s),\\
			& W_{\theta^{(i, j')}}(s)\ge \left[\mathcal{T}_{\pi'_{\eta^{(i, j')}}}W_{\theta^{(i, j')}}\right](s),\\
		\end{aligned}
	\end{equation*}
	for $\forall 0\le j'\le j$.
\end{lemma}
\begin{proof}
	The proof is similar to Lemma \ref{lem1}.
\end{proof}
\begin{lemma}\label{lemma2}
	With at least probability $1 - \delta$, Algorithm \ref{alg2} for $\mM'$ will output $\eta^{(R, R')}$ which satisfies $\|V_{\eta^{(R, R')}}' - v'\|\le \epsilon$, together with events $\mG^{(0)}_{1}, \mG^{(i)}_{1}, \mE^{(i-1, j)}_{1}$ for $1\le i\le R', 0\le j\le R$.
\end{lemma}
\begin{proof}
	The proof is similar to Lemma \ref{lemma1}.
\end{proof}

\par Our next lemma indicates that if $V_{\theta}$ and $W_{\eta}$ are both $\epsilon$-optimal values, then the strategy obtained from Algorithm \ref{alg3} is an $\epsilon$-optimal strategy.

\begin{lemma}\label{combine}
	If $V_{\theta^{(R', R)}}$ and $W_{\eta^{(R', R)}}$ are both $\epsilon$-optimal values, where $\theta^{(R', R)}, \eta^{(R', R)}$ are parameters obtained from Algorithm \ref{alg2} with inputs $\mM$ and $\mM'$, respectively. and  $\mG^{(0)}, \mG^{(i)}, \mE^{(i-1, j)}$, $\mG^{(0)}_{1}, \mG^{(i)}_{1}, \mE^{(i-1, j)}_{1}$ all hold for $\forall 1\le i\le R', 0\le j\le R$, then the strategy $\pi = (\pi_{1}, \pi_{2}$) output from Algorithm \ref{alg3} is an $\epsilon$-optimal strategy.
\end{lemma}
\begin{proof}
	\par We define the following operators mapping from value functions to value functions.
	\begin{equation*}
		\begin{aligned}
			\ [\mT_{\max, \pi_{2}}V](s) = \begin{cases}
				\max_{a\in\mathcal{A}_{s}}[r(s, a) + \gamma P(\cdot|s, a)^{T}]V, \quad\forall s\in\mS_{1},\\
				r(s, \pi_{2}(s)) + \gamma P(\cdot|s, \pi_{2}(s))^{T}V, \quad\forall s\in\mS_{2},
			\end{cases}\\
			[\mT_{\pi_{1}, \min}V](s) = \begin{cases}
				r(s, \pi_{1}(s)) + \gamma P(\cdot|s, \pi_{1}(s))^{T}V, \quad\forall s\in\mS_{1},\\
				\min_{a\in\mathcal{A}_{s}}[r(s, a) + \gamma P(\cdot|s, a)^{T}]V, \quad\forall s\in\mS_{2}.
			\end{cases}
		\end{aligned}
	\end{equation*}
	Then $\mT_{\max, \pi_{2}}, \mT_{\pi_{1}, \min}$ are both monotonic and contracting operators with contraction factor $\gamma$, and it is easy to see that $V^{\overline{\pi}_{1}, \pi_{2}}, V^{\pi_{1}, \overline{\pi}_{2}}$ are fixed points of $\mT_{1}, \mT_{2}$, respectively, where $\overline{\pi}_{1}, \overline{\pi}_{2}$ are optimal counterstrategies against $\pi_{2}, \pi_{1}$.
	\par We next prove that $V^{\overline{\pi}_{1}, \pi_{2}}$ satisfies
	\begin{equation}\label{ineq2}
		v^{*}\le V^{\overline{\pi}_{1}, \pi_{2}}\le W_{\eta^{(R', R)}}.
	\end{equation}
	According to Lemma \ref{lem5}, for any $s\in\mS_{1}$, we have
	\begin{equation*}
		\left[\mT_{\max, \pi_{2}}W_{\eta^{(R', R)}}\right](s) = \max_{a\in\mA_{s}}r(s, a) + \gamma P(\cdot|s, a)^{T}W_{\eta^{(R', R)}} = \left[\mathcal{T}W_{\eta^{(R', R)}}\right](s)\le W_{\eta^{(R', R)}}(s),
	\end{equation*}
	and for any $s\in\mS_{2}$, we have
	\begin{equation*}
		\begin{aligned}
			\left[\mT_{\max, \pi_{2}}W_{\eta^{(R', R)}}\right](s) & = r(s, \pi_{2}(s)) + \gamma P(\cdot|s, \pi_{2}(s))^{T}W_{\eta^{(R', R)}}\\
			& = \left[\mT_{\pi'_{\eta^{(R', R)}}}W_{\eta^{(R', R)}}\right](s)\le W_{\eta^{(R', R)}}(s).
		\end{aligned}
	\end{equation*}
	These inequalities indicates that $\mT_{\max, \pi_{2}}W_{\eta^{(R', R)}}\le W_{\eta^{(R', R)}}$. Hence according to the monotonicity of $\mT_{\max, \pi_{2}}$, we have $V^{\overline{\pi}_{1}, \pi_{2}}\le W_{\eta}^{(R', R)}(s)$. Moreover, since $\overline{\pi}_{1}$ is an optimal counterstrategy against $\pi_{2}$, we have $v^{*}\le V^{\overline{\pi}_{1}, \pi_{2}}$. The inequalities \eqref{ineq2} has been proved. 
	\par Next noticing that $\|W_{\eta^{(R', R)}} - v^{*}\|\le \epsilon$,  we have
	\begin{equation*}
		\|V^{\overline{\pi}_{1}, \pi_{2}} - v^{*}\|\le \epsilon.
	\end{equation*}
	Similarly we have $\|V^{\pi_{1}, \overline{\pi}_{2}} - v^{*}\|\le \epsilon$ considering the operator $\mT_{\pi_{1}, \min}$. These two inequalities together indicate that $\pi$ is an $\epsilon$-optimal strategy.
\end{proof}

\subsection{Proof of Theorem \ref{thm2}}
\begin{proof}[Proof of Theorem \ref{thm2}]
	\par According to Lemma \ref{lemma1} and \ref{lemma2}, the event that $V_{\theta^{(R, R')}}, W_{\eta^{(R, R')}}$ are both $\epsilon$-optimal values, together with events $\mG^{(0)}, \mG^{(i)}, \mE^{(i, j)}, \mG^{(0)}_{1}, \mG^{(i)}_{1}, \mE^{(i, j)}_{1}$ for $1\le i\le R', 0\le j\le R$, happen with probability at least $1 - 2\delta$. Hence according to Lemma \ref{combine}, the output $\pi$ of Algorithm \ref{alg3} is $\epsilon$-optimal strategy with probability at least $1 - 2\delta$. The total samples used in our algorithm is
	\begin{equation*}
		2(R'RKm_{1} + R'Km) = \tilde{\mathcal{O}}\left(\frac{KL^{2}}{\epsilon^{2}(1-\gamma)^{4}}\right)
	\end{equation*}
	samples.
\end{proof}

\section{Proof of Theorem \ref{thm4}}
\par We first present a proposition indicating that an approximate optimal strategy of $\mM = (\mS, \mA, P, r, \gamma)$ is also an approximate optimal strategy of $\mM' = (\mS, \mA, P', r, \gamma)$.
\begin{proposition}\label{prop2}
	Suppose $P, \tilde{P}$ are two transition models such that
	\begin{equation*}
		\|P(\cdot|s, a) - P'(\cdot|s, a)\|_{TV}\le \xi, \quad \forall (s, a)\in\mS\times\mA.
	\end{equation*}
	Then for two 2-TBSGs $\mM = (\mS, \mA, P, r, \gamma), \mM' = (\mS, \mA, P', r, \gamma)$, if $\pi$ is an $\epsilon$-optimal strategy of $\mM$, $\pi$ is also an $\left(\frac{2\xi}{(1-\gamma)^{2}} + 2\epsilon\right)$-optimal strategy of $\mM'$.
\end{proposition}
\begin{proof}
	Suppose $\pi = (\pi_{1}, \pi_{2})$, $\overline{\pi}_{1}, \overline{\pi}_{2}$ are optimal counterstrategies against $\pi_{2}, \pi_{1}$ in $\mM$, and $\overline{\pi}_{1}', \overline{\pi}_{2}'$ are optimal counterstrategy against $\pi_{2}, \pi_{1}$ in $\mM'$. We also assume that $V, U$ are value functions of $\mM$ and $\mM'$, respectively.
	\par According to the TV condition, for any strategy $\pi$ we have
	\begin{equation*}
		\begin{aligned}
			\|V^{\pi} - U^{\pi}\|_{\infty} & = \|(I - \gamma P_{\pi})^{-1}r_{\pi} - (I - \gamma P_{\pi}')^{-1}r_{\pi}\|_{\infty}\\
			& = \|(I - \gamma P_{\pi}')^{-1}(\gamma P_{\pi} - \gamma P_{\pi}')(I - \gamma P_{\pi})^{-1}r_{\pi}\|_{\infty}\\
			& \le \|(I - \gamma P_{\pi}')^{-1}\|_{\infty}\|\gamma P_{\pi} - \gamma P_{\pi}'\|_{\infty}\|(I - \gamma P_{\pi})^{-1}\|_{\infty}\|r_{\pi}\|_{\infty}\\
			& \le \frac{1}{1-\gamma \|P_{\pi}'\|_{\infty}}\cdot\|\gamma P_{\pi} - \gamma P_{\pi}'\|_{\infty}\cdot\frac{1}{1-\gamma \|P_{\pi}\|_{\infty}}\|r_{\pi}\|_{\infty}\\
			& \le \frac{\xi}{(1-\gamma)^{2}},
		\end{aligned}
	\end{equation*}
	where the last inequality follows the facts
	\begin{equation*}
	\|P_{\pi}\|_{\infty} = \|P_{\pi}'\|_{\infty} = 1, \quad \|r_{\pi}\|_{\infty}\le 1,\quad \|P_{\pi} - P_{\pi}'\|_{\infty} = \max_{s\in\mS}\|P(\cdot|s, \pi(s)) - P'(\cdot|s, \pi(s))\|_{TV}\le \xi.
	\end{equation*}
	Next, since $\overline{\pi}_{2}, \overline{\pi}_{2}'$ are optimal counterstrategies against $\pi_{1}$ in $\mM$ and $\mM'$, we have
	\begin{equation*}
		V^{\pi_{1}, \overline{\pi}_{2}}\le V^{\pi_{1}, \overline{\pi}_{2}'}, \quad U^{\pi_{1}, \overline{\pi}_{2}}\ge U^{\pi_{1}, \overline{\pi}_{2}'}.
	\end{equation*}
	Hence for any $s\in\mS$,
	\begin{equation*}
		- \frac{\xi}{(1-\gamma)^{2}}\le V^{\pi_{1}, \overline{\pi}_{2}}(s) - U^{\pi_{1}, \overline{\pi}_{2}}\le V^{\pi_{1}, \overline{\pi}_{2}}(s) - U^{\pi_{1}, \overline{\pi}_{2}'}(s)\le V^{\pi_{1}, \overline{\pi}_{2}'}(s) - U^{\pi_{1}, \overline{\pi}_{2}'}(s)\le \frac{\xi}{(1-\gamma)^{2}},
	\end{equation*}
	which indicates that
	\begin{equation*}
		\left\|V^{\pi_{1}, \overline{\pi}_{2}} - U^{\pi_{1}, \overline{\pi}_{2}'}\right\|_{\infty}\le \frac{\xi}{(1-\gamma)^{2}}.
	\end{equation*}
	\par Similarly, we have
	\begin{equation*}
		\left\|V^{\overline{\pi}_{1}, \pi_{2}} - U^{\overline{\pi}_{1}', \pi_{2}}\right\|_{\infty}\le \frac{\xi}{(1-\gamma)^{2}}.
	\end{equation*}
	\par Moreover, since $\pi$ is an $\epsilon$-optimal strategy of $\mM$, we have
	\begin{equation*}
		\|V^{\pi_{1}, \overline{\pi}_{2}} - V^{\overline{\pi}_{1}, \pi_{2}}\|_{\infty}\le \|V^{\pi_{1}, \overline{\pi}_{2}} - v^{*}\| + \|V^{\overline{\pi}_{1}, \pi_{2}} - v^{*}\|_{\infty}\le 2\epsilon,
	\end{equation*}
	where $v^{*}$ is the equilibrium value of $\mM'$. This inequality, together with above two inequalities, indicates that
	\begin{equation*}
		\left\|U^{\pi_{1}, \overline{\pi}_{2}'} - U^{\overline{\pi}_{1}', \pi_{2}}\right\|\le \frac{2\xi}{(1-\gamma)^{2}} + 2\epsilon.
	\end{equation*}
	Next noting that
	\begin{equation*}
		U^{\pi_{1}, \overline{\pi}_{2}'} \le u^{*}\le U^{\overline{\pi}_{1}', \pi_{2}},
	\end{equation*}
	where $u^{*}$ is the equilibrium value of $\mM'$, we have
	\begin{equation*}
		\left\|U^{\pi_{1}, \overline{\pi}_{2}'} - u^{*}\right\|_{\infty}\le \frac{2\xi}{(1-\gamma)^{2}} + 2\epsilon, \quad \left\|U^{\overline{\pi}_{1}', \pi_{2}} - u^{*}\right\|_{\infty}\le \frac{2\xi}{(1-\gamma)^{2}} + 2\epsilon.
	\end{equation*}
	These two inequalities indicate that $\pi$ is an $\left(\frac{2\xi}{(1-\gamma)^{2}} + 2\epsilon\right)$-optimal strategy of $\mM'$.
\end{proof}

\begin{proof}[Proof of Theorem \ref{thm4}]
\par Since Algorithm \ref{alg1}, \ref{alg2} and \ref{alg3} only sample from $P(\cdot|s, a)$ for $(s, a)\in \mK$, and $P$ and $P'$ agree on $\mK$, the results of these algorithms executing on $\mM = (\mS, \mA, P, r, \gamma)$ are same as the results of these algorithms executing on $\mM' = (\mS, \mA, P', r, \gamma)$. According to Theorem \ref{thm1} and Theorem \ref{thm2}, with probability at least $1 - \delta$, Algorithm \ref{alg1} outputs $w^{(R)}$ such that $\pi_{w^{(R)}}$ is an $\epsilon$-optimal strategy of $\mM$, and with probability at least $1 - 2\delta$, Algorithm \ref{alg3} outputs an $\epsilon$-optimal strategy $\pi$ of $\mM$. Therefore, according to Proposition \ref{prop2}, $\pi_{w^{(R)}}$ and $\pi$ are $\left(\frac{2\xi}{(1-\gamma)^{2}} + 2\epsilon\right)$-optimal strategies of $\mM'$. The proof is completed.
\end{proof}

\end{document}